\documentclass{article}

    \PassOptionsToPackage{numbers, compress}{natbib}
\usepackage[preprint]{neurips_2026}

\usepackage[utf8]{inputenc} 
\usepackage[T1]{fontenc}    
\usepackage{hyperref}       
\usepackage{url}            
\usepackage{booktabs}       
\usepackage{amsfonts}       
\usepackage{amsmath}        
\usepackage{amssymb}        
\usepackage{amsthm}         
\usepackage{mathtools}      
\usepackage{nicefrac}       
\usepackage{microtype}      
\usepackage{xcolor}         
\usepackage{graphicx}       
\usepackage{subfigure}      
\usepackage{algorithm}      
\usepackage{algorithmic}    
\usepackage{wrapfig}        
\usepackage{multirow}       
\usepackage{enumitem}       
\usepackage{caption}        
\usepackage{adjustbox}      

\newtheorem{proposition}{Proposition}
\newtheorem*{proposition*}{Proposition}

\newtheorem*{claim*}{Claim}

\newtheorem{remark}{Remark}
\newtheorem{observation}{Observation}

\newcommand{\ours}{LGR}
\newcommand{\ourstopk}{LGR (topk)}
\newcommand{\rkl}{\text{R-KL}}

\DeclareMathOperator{\Ent}{\mathcal{H}}

\title{Your Teacher Can't Help You Here: Combating Supervision Fidelity Decay in On-Policy Distillation}

\author{%
  Yanjiang Liu\textsuperscript{1,2} \quad
  Jie Lou\textsuperscript{3} \quad 
  Xinyan Guan\textsuperscript{1,2} \quad 
  Yuqiu Ji\textsuperscript{3} \quad
  Hongyu Lin\textsuperscript{2} \quad 
  Ben He\textsuperscript{1,2} \quad \\
  \textbf{Xianpei Han\textsuperscript{2}}
  \textbf{Le Sun\textsuperscript{2}}
  \textbf{Xing Yu\textsuperscript{3}}
  \textbf{Yaojie Lu\textsuperscript{2}}
  \vspace{5pt}\\
\textsuperscript{\rm 1}University of Chinese Academy of Sciences \\
\textsuperscript{\rm 2}Chinese Information Processing Laboratory \\Institute of Software, Chinese Academy of Sciences\\
    \textsuperscript{\rm 2}University of Chinese Academy of Sciences, Beijing, China \\
     \textsuperscript{\rm 3}Xiaohongshu
}
\begin{document}

\maketitle
{\let\thefootnote\relax\footnotetext{Email: \texttt{liuyanjiang22@mails.ucas.ac.cn}, \texttt{\{luyaojie,hongyu,xianpei,sunle\}@iscas.ac.cn}}}
{\let\thefootnote\relax\footnotetext{Codes are available at \url{https://github.com/zui-jiang/LGR}.}}

\begin{abstract}  \label{sec:abs}

On-policy distillation transfers reasoning capabilities by training a student model on its own generated trajectories using token-level feedback from a teacher. 
However, we identify a critical bottleneck, \textbf{Supervision Fidelity Decay (SFD)}: as student-generated prefixes lengthen, the teacher’s next-token distribution becomes less confident and less discriminative. 
Consequently, the teacher-dependent corrective signal in reverse-KL distillation weakens, causing student drift to compound across long reasoning chains. 
To mitigate SFD, we introduce \textbf{Lookahead Group Reward (\ours{})}.
Building on the insight that next-step teacher confidence reflects the discriminative strength of future reverse-KL supervision, \ours{} evaluates the student’s top-K candidate tokens by the teacher confidence they induce at the subsequent step and assigns a group-normalized reward.
To maintain computational efficiency, we further design an entropy-triggered tree-attention mechanism.
Across six math and code benchmarks, \ours{} improves mean@8 by \textbf{2.57} points over OPD for a 7B student, with gains increasing in longer-generation and reaching +\textbf{4.92} points on AIME-26 at 39k tokens.
\end{abstract}

\section{Introduction} \label{sec:intro}

Large language models (LLMs) with reasoning capabilities have achieved remarkable performance on complex mathematical and coding tasks. Recent advance models~\citep{openai2026openaio1card,Guo_2025,deepseekai2026deepseekv4} demonstrate that extended reasoning chains can unlock capabilities previously thought to require much larger models. On-policy distillation (OPD), where a student model generates its own reasoning trajectories and learns from the teacher's token-level feedback, has emerged as a primary paradigm for transferring these capabilities to efficient, deployable models~\citep{agarwal2024onpolicydistillationlanguagemodels,lu2025onpolicydistillation,gu2026minillmonpolicydistillationlarge,patiño2025_unlocking_on_policy_distillation_for_any_model_family,yang2025qwen3technicalreport,coreteam2026mimov2flashtechnicalreport,yang2026nemotroncascade2posttrainingllms}.

However, current OPD methods~\citep{agarwal2024onpolicydistillationlanguagemodels,lu2025onpolicydistillation,gu2026minillmonpolicydistillationlarge,patiño2025_unlocking_on_policy_distillation_for_any_model_family,ko2026scalingreasoningefficientlyrelaxed} implicitly treat the teacher as a \emph{static oracle} whose supervision quality is unaffected by what the student generates. We challenge this assumption and reveal a critical failure mode: as the student generates increasingly long sequences, its outputs progressively deviate from the teacher's training distribution, causing the teacher's supervision quality to \emph{degrade monotonically}, which we refer to as \textbf{Supervision Fidelity Decay (SFD)}. 

\begin{center}
\textbf{\textit{Does teacher supervision quality remain reliable throughout long reasoning chains --- and if not, how can we actively maintain it?}}
\end{center}

 As an initial observation, training OPD with varying maximum generation lengths (Figure~\ref{fig:maxlen_motivation}) reveals that performance improves from 3k to 9k tokens and plateaus around 16k, followed by a significant decline at 39k. This trend indicates a potential failure mode during extremely long generations. To isolate the cause, we design a controlled prefix-completion experiment (Figure~\ref{fig:sfd_evidence}): as the student prefix grows longer, both the teacher's downstream task accuracy and peak next-token probability decay monotonically, while the teacher on its own prefixes maintains significantly higher confidence, which indicating that SFD stems from student drift. The subplots further show that teacher confidence jumps immediately when the teacher takes over, confirming that different token choices lead to different teacher confidence one step ahead.

\begin{figure}[t]
    \begin{minipage}[t]{0.34\textwidth}
        \input{figure/fig2}
    \end{minipage}
    \hfill
    \begin{minipage}[t]{0.64\textwidth}
        \input{figure/fig1}
    \end{minipage}
    \vspace{-1em}
\end{figure}

Through theoretical analysis, we find that the declining max-prob directly collapses the reverse-KL gradient. As teacher confidence falls, its log-probability varies less across token choices, effectively reducing the learning signal to a student-only signal that reinforces existing modes without correction. Yet observations from the subplot suggest a solution; even at the same out-of-distribution context, teacher confidence at the subsequent position differs across token choices. By the same gradient analysis, higher next position confidence implies a more discriminative future signal: choosing the token that maximizes the next position's confidence directly preserves future supervision quality. We operationalize this as \textbf{Lookahead Group Reward (LGR)}, a group-normalized reward over the student's top-$K$ candidates, where this group normalization removes the high-variance absolute confidence level and retains only the relative ranking across candidates. To maintain computational efficiency, we further design a tree-attention mechanism triggered by entropy.

Our key contributions are:

\begin{itemize}[leftmargin=1.5em, itemsep=1pt, topsep=1pt]
    \item \textbf{We identify SFD as a fundamental failure mode of OPD.} We show that teacher accuracy and peak confidence decay as student prefix length increases, a process that we prove collapses the reverse KL gradient into a signal that reinforces itself based solely on student outputs.

    \item \textbf{We propose a principled remedy that looks one step ahead.} Since higher teacher confidence at the next position provides more discriminative future supervision, \ours{} selects tokens using group normalized rewards and an efficient tree attention mechanism triggered by entropy.
    
    \item \textbf{We show \ours{}'s gains grow with reasoning length.} \ours{} significantly outperforms OPD and alternative distillation methods, with pronounced gains on long reasoning tasks.
\end{itemize}

\section{Supervision Fidelity Decays Along Student Trajectories} \label{sec:analysis}
\vspace{-1em}

\subsection{On-Policy Reverse-KL as Policy Gradient} \label{sec:background}
\vspace{-1em}

Let $\pi_T$ denote a teacher model and $\pi_\theta$ a student model parameterized by $\theta$. Given a prompt $\mathbf{c}$, the student autoregressively generates a sequence $\mathbf{x} = (x_1, x_2, \ldots, x_L)$. On-policy reverse-KL distillation minimizes~\citep{agarwal2024onpolicydistillationlanguagemodels,xiao2025connectionimitationlearningrlhf}:
\begin{equation} \label{eq:rkl}
    \mathcal{L}_{\rkl}(\theta) = \mathbb{E}_{\mathbf{x} \sim \pi_\theta(\cdot | \mathbf{c})} \left[ \sum_{t=1}^{L} \log \frac{\pi_\theta(x_t | \mathbf{x}_{<t}, \mathbf{c})}{\pi_T(x_t | \mathbf{x}_{<t}, \mathbf{c})} \right].
\end{equation}
Following standard practice in on-policy distillation~\citep{lu2025onpolicydistillation}, we detach the generated sequence $\mathbf{x}$ from the computation graph during loss computation. Under this stop-gradient assumption, the per-position KL gradient decomposes independently, yielding the policy gradient form:
\begin{equation} \label{eq:pg}
    \nabla_\theta \mathcal{L}_{\rkl} = \mathbb{E}_{\mathbf{x} \sim \pi_\theta} \left[ \sum_{t=1}^{L} \nabla_\theta \log \pi_\theta(x_t | \mathbf{x}_{<t}) \cdot A_t \right],
\end{equation}
where the per-token advantage is:
\begin{equation} \label{eq:advantage}
    A_t = 1 + \log \pi_\theta(x_t | \mathbf{x}_{<t}) - \log \pi_T(x_t | \mathbf{x}_{<t}).
\end{equation}
This reveals that on-policy reverse-KL is equivalent to a policy gradient in the style of REINFORCE. Equivalently maximizing the per-token reward $r_t = \log\pi_T(x_t|\mathbf{x}_{<t}) - \log\pi_\theta(x_t|\mathbf{x}_{<t})$~\citep{Sutton1998,NEURIPS2023_a85b405e}. We refer to this as \textbf{local supervision}, in which the teacher provides a distributional target at each position.

\subsection{The Supervision Capability Functional and Supervision Fidelity Decay} \label{sec:sfd}

\textbf{Supervision as a functional.} The student policy $\pi_\theta$ generating a sequence of length $L$ induces a \emph{state visitation distribution} $\rho_{\pi_\theta}^{L}(c)$ over prefix contexts $c$. The teacher's supervision capability is then a \textbf{functional} of the student policy:
\begin{equation} \label{eq:supervision_functional}
    \mathcal{C}[\pi_\theta] = \mathbb{E}_{c \sim \rho_{\pi_\theta}^L}\!\left[f_T(c)\right], \quad \text{where} \quad f_T(c) = \max_{v \in \mathcal{V}} \pi_T(v | c)
\end{equation}
 is the teacher's local supervision quality at state $c$. This functional captures a key asymmetry. While the teacher remains unchanged~\citep{agarwal2024onpolicydistillationlanguagemodels,lu2025onpolicydistillation,gu2026minillmonpolicydistillationlarge,patiño2025_unlocking_on_policy_distillation_for_any_model_family,yang2025qwen3technicalreport,coreteam2026mimov2flashtechnicalreport}, the student's state visitation distribution  $\rho_{\pi_\theta}^L$ that shifts, dragging the integration domain into regions where $f_T$ is low.

\textbf{Position-dependent SFD curve.} We define the position-dependent version:
\begin{equation} \label{eq:sfd_curve}
    \mathcal{C}^{(t)}[\pi_\theta] = \mathbb{E}_{c \sim \rho_{\pi_\theta}^t}\!\left[\max_{v} \pi_T(v | c)\right],
\end{equation}
so that $\mathcal{C}^{(t)}[\pi_\theta]$ as a function of $t$ traces the SFD curve. The aggregate $\mathcal{C}[\pi_\theta] = \mathbb{E}_t[\mathcal{C}^{(t)}[\pi_\theta]]$ summarizes total supervision quality. When $L$ is short, $\rho_{\pi_\theta}^L$ remains within the teacher's training manifold and $f_T(c)$ stays high; when $L$ is long, autoregressive drift causes $\rho_{\pi_\theta}^L$ to shift into the teacher's OOD region, where $f_T(c)$ collapses.

\textbf{Supervision fidelity as a downstream metric.} For a specific student-generated prefix $\mathbf{x}_{<t}$, the teacher's \emph{supervision fidelity} at position $t$ is:
\begin{equation} \label{eq:fidelity}
    \mathcal{F}(t) \coloneqq \max_{v \in \mathcal{V}} \pi_T(v | \mathbf{x}_{\leq t}, \mathbf{c}),
\end{equation}
i.e., the teacher's peak next-token probability after observing the student's prefix up to $t$. Note that $\mathcal{C}^{(t)}[\pi_\theta] = \mathbb{E}[\mathcal{F}(t)]$ is the expectation of $\mathcal{F}(t)$ over student trajectories. At the task level, we can also measure supervision fidelity as $P(\text{teacher completes correctly from position } t \mid \mathbf{x}_{<t}^{\theta})$, which correlates strongly with $\mathcal{F}(t)$, thereby validating that max-$p$ is an effective proxy for downstream supervision quality.

\textbf{Empirical observation.} Using DeepSeek-R1-Distill-Qwen-1.5B~\citep{Guo_2025} as the student and 32B~\citep{Guo_2025} as the teacher on AIME~\citep{aime24}, we generate student prefixes of varying lengths and hand off to the teacher. Figure~\ref{fig:sfd_evidence} shows: (1) teacher completion accuracy drops monotonically with prefix length; (2) $\bar{\mathcal{F}}$ decreases correspondingly, validating max-$p$ as a proxy; (3) the teacher on its own prefixes maintains significantly higher fidelity, confirming SFD stems from student drift. The inset subplots further show that at the handoff point, teacher confidence \emph{jumps immediately}, which demonstrates that token choices at position $t$ do affect teacher confidence at $t{+}1$, even under severe SFD. The jump shrinks at longer prefixes (${\sim}14$k vs.\ ${\sim}2$k) and after OPD optimization, consistent with accumulated drift. This directly motivates Section~\ref{sec:confidence}.

\subsection{Theoretical Analysis: Signal Collapse and Compounding Drift} \label{sec:theory}

We formalize the impact of SFD on gradient quality. As the teacher's distribution becomes diffuse, its discriminative contribution vanishes; consequently, this single-position failure compounds across positions into a self-reinforcing drift.

\begin{proposition}[Teacher signal vanishing under SFD] \label{prop:snr}
Define $\Delta_T(t) \coloneqq \mathrm{Var}_{x_t \sim \pi_\theta}[\log \pi_T(x_t | \mathbf{x}_{<t})]$. Then $\Delta_T(t) \leq \log^2|\mathcal{V}| - \mathrm{Ent}(\pi_T(\cdot|\mathbf{x}_{<t}))^2$, so as the teacher's distribution becomes diffuse under SFD, $\Delta_T(t) \to 0$ and $\mathrm{SNR}_T(t) = O(\Delta_T(t)) \to 0$: the teacher's discriminative contribution to the gradient vanishes entirely, leaving a student-only signal that reinforces existing modes without correction. \emph{(Proof: Appendix~\ref{app:proof:snr}.)}
\end{proposition}

Once the teacher signal vanishes at position $t$, the student selects tokens from its own sharpened distribution, pushing the next context further out-of-distribution, thereby further degrading the signal at $t{+}1$.

\begin{proposition}[Self-reinforcing drift under reverse-KL] \label{prop:drift}
Let $d_t \coloneqq D(\pi_T(\cdot|\mathbf{x}_{<t}^{\theta}), \pi_T(\cdot|\mathbf{x}_{<t}^{*}))$ be the distributional drift at position $t$, and assume $\Delta_T(t)$ is non-increasing in $d_t$ (greater drift degrades teacher discriminability, as implied by SFD). When $\Delta_T(t) < \delta_{\mathrm{crit}}$, the teacher-free advantage reinforces the student's existing modes, causing $\mathbb{E}[d_{t+1}|d_t] \geq d_t$: drift compounds across positions, creating a positive feedback loop that degrades supervision irreversibly. Forward-KL avoids this by construction ($d_t = 0$ always) but introduces exposure bias. \emph{(Proof sketch: Appendix~\ref{app:proof:drift}.)}
\end{proposition}

\subsection{Training Dynamics: The Vicious Cycle and Supervision Boundary Contraction} \label{sec:vicious_cycle}

The functional perspective reveals a deeper consequence. Define the \emph{effective learning horizon} $t_{\mathrm{eff}}(k) = \sup\{t : \mathcal{C}^{(t)}[\pi_{\theta_k}] > C_{\min}\}$ as the furthest position where teacher supervision exceeds a minimum useful threshold at training step $k$.

Under the self-reinforcing drift established in Proposition~\ref{prop:drift}, training induces a \textbf{vicious cycle}: positions beyond $t_{\mathrm{eff}}$ receive no useful supervision ($\ell \approx 0$) and therefore do not improve; the student's uncorrected behavior at these positions continues to push the teacher further out-of-distribution, which in turn may cause $t_{\mathrm{eff}}$ to \emph{shrink} at the next training step. This creates a ``learning desert'' that expands over training, establishing a \textbf{reasoning length ceiling} $t^*$ beyond which the student can never improve under standard on-policy distillation.

The $\mathcal{C}^{(t)}$ curve follows a sigmoidal decay: the early portion maintains high supervision quality, the late portion collapses, and the transition sharpens progressively over training. We verify this empirically: fitting a logistic sigmoid to both curves in Figure~\ref{fig:sfd_evidence} yields $R^2 > 0.997$, and the fitted parameters show that OPD training contracts the supervision boundary ($t^*$: $7.43 \to 7.04$k) and steepens the transition slope ($\varepsilon$: $0.21 \to 0.26$). Notably, all four parameter shifts align with 
 the direction predicted by the vicious cycle (see Appendix~\ref{app:sigmoid} for full fit results).

This compounding effect is fundamental to on-policy reverse-KL and cannot be resolved by simply adjusting learning rates or adding regularization, as these measures only slow the drift rate $\varepsilon$ without changing the structural problem. Instead, it motivates directly optimizing the supervision functional $\mathcal{C}[\pi_\theta]$ itself, which provides a structurally different signal that steers the student toward states where the teacher can provide high-quality supervision.


\section{Lookahead Confidence as a Supervision Signal} \label{sec:method}

\subsection{From Gradient Failure to One-Step Lookahead} \label{sec:two_layers}

The preceding analysis shows that SFD collapses the gradient: as $\pi_T(\cdot|\mathbf{x}_{<t})$ becomes diffuse, $A_t = 1 + \log\pi_\theta - \log\pi_T$ loses discriminability and reduces to a student-only signal (Proposition~\ref{prop:snr}). The standard objective (Eq.~\ref{eq:rkl}) is thus \textbf{agnostic to state quality}: it still forces the student to match the teacher's now-uninformative distribution.

\textbf{One-step-ahead retains discriminability.} Even when $\pi_T(\cdot|\mathbf{x}_{<t})$ is near-uniform, different token choices $x_t^{(k)}$ steer the context into different next-states, and the teacher's distributions at $t{+}1$ can differ substantially across candidates. Instead of asking what the teacher wants at position $t$ (unanswerable under SFD), we ask which candidate causes the \emph{least further drift}, a question that remains informative even when local supervision has collapsed.

\textbf{Breaking the vicious cycle.} This signal directly targets Proposition~\ref{prop:drift}: by favoring tokens that maintain higher teacher confidence at $t{+}1$, the student slows the per-step drift rate, preventing $t_{\mathrm{eff}}$ from contracting. The concrete realization is a per-token one-step-ahead confidence reward (Section~\ref{sec:confidence}), which provides a greedy one-step approximation to $\nabla_\theta \mathcal{C}[\pi_\theta]$, sufficient to slow drift without requiring full multi-step lookahead. We now formalize when this approximation retains useful discriminability.

\begin{proposition}[One-step-ahead discriminability survives local supervision failure] \label{prop:one_step}
Define the one-step-ahead discriminability at position $t$ as:
\begin{equation}
    D_{\text{ahead}}(t) \coloneqq \max_{k, k'} \left| \max_v \pi_T(v | \mathbf{x}_{<t}, x_t^{(k)}) - \max_v \pi_T(v | \mathbf{x}_{<t}, x_t^{(k')}) \right|,
\end{equation}
which measures the maximum difference in teacher confidence at $t{+}1$ across candidate token choices at $t$. Then $D_{\mathrm{ahead}}(t)$ is \textbf{independent} of the teacher's local entropy $\Ent(\pi_T(\cdot|\mathbf{x}_{<t}))$: even when $\pi_T(\cdot|\mathbf{x}_{<t})$ is uniform (maximum SFD, $\Delta_T(t)=0$), $D_{\mathrm{ahead}}(t)$ can be arbitrarily large. \emph{(Proof: Appendix~\ref{app:proof:one_step}.)}
\end{proposition}

This proposition is the theoretical foundation for our confidence reward: even when local supervision (Proposition~\ref{prop:snr}) has completely failed, the one-step-ahead comparison can still extract useful guidance from the teacher.

\subsection{Confidence Reward Design} \label{sec:confidence}

We operationalize the one-step-ahead comparison as a reward signal. At each position $t$, after the student samples $x_t \sim \pi_\theta(\cdot | \mathbf{x}_{<t})$, we evaluate the teacher's supervision fidelity at position $t{+}1$:
\begin{equation} \label{eq:raw_confidence}
    r_{\text{raw}}(x_t) = \max_{v \in \mathcal{V}} \pi_T(v | \mathbf{x}_{<t}, x_t, \mathbf{c}).
\end{equation}
This measures the teacher's peak next-token probability \emph{after} observing the student's action $x_t$. The critical distinction is temporal: local supervision (R-KL) evaluates the teacher's state at $t$ given context $\mathbf{x}_{<t}$, while the confidence reward evaluates the teacher's state at $t{+}1$ given context $(\mathbf{x}_{<t}, x_t)$. By comparing this quantity across $K$ candidates, we identify which token choice at $t$ \textbf{causes the least further degradation} in teacher supervision quality at $t{+}1$. This selection seeks not to recover an in-distribution state, but rather to minimize the next step of SFD progression.

\begin{proposition}[Max-$p$ as relative drift indicator] \label{prop:maxp_proximity}
Let $P^*_T$ denote the teacher's next-token distribution conditioned on an in-distribution prefix, and let $P_T^{(\mathbf{x})}$ denote the distribution conditioned on a student-generated prefix $\mathbf{x}$. If the teacher is $\beta$-smooth (small perturbations in context produce bounded distributional shifts), then:
\begin{equation}
    \max_v P^*_T(v) - \max_v P_T^{(\mathbf{x})}(v) \leq \|P^*_T - P_T^{(\mathbf{x})}\|_\infty \leq \beta \cdot d(\mathbf{x}, \mathcal{X}_T),
\end{equation}
where $d(\mathbf{x}, \mathcal{X}_T)$ measures the distance from the teacher's in-distribution manifold $\mathcal{X}_T$. That is, \textbf{higher max-$p$ implies closer proximity to the teacher's competent region}. \emph{(Proof: Appendix~\ref{app:proof:maxp}.)}
\end{proposition}

\begin{remark}[Max-$p$ vs.\ entropy]
Maximizing $\max_v \pi_T(v|\cdot)$ is directionally equivalent to minimizing the teacher's prediction entropy $\Ent(\pi_T(\cdot|\cdot))$. We prefer max-$p$ for practical reasons: (i) it requires only an argmax rather than a full softmax and log-sum, (ii) its values lie naturally in $[0, 1]$, and (iii) under group normalization, the candidate ranking is consistent between the two.
\end{remark}

\textbf{Group normalization.} Since $r_{\text{raw}}$ varies in absolute magnitude across positions and tasks, we normalize within the top-$K$ student candidates (ranked by $\pi_\theta$), inspired by GRPO~\citep{Guo_2025}:
\begin{equation} \label{eq:group_norm}
    r_{\text{conf}}(x_t) = \frac{r_{\text{raw}}(x_t) - \mu_K}{\sigma_K + \epsilon},
\end{equation}
where $\mu_K$, $\sigma_K$ are the group mean and std over $\{r_{\text{raw}}^{(k)}\}_{k=1}^K$. This removes absolute-scale variance while preserving relative ranking across candidates (Appendix~\ref{app:proof:groupnorm}).

\textbf{Combined loss.} Adding the confidence reward to the per-token advantage (Eq.~\eqref{eq:advantage}) gives the \ours{} loss:
\begin{equation} \label{eq:lgr_loss}
    \mathcal{L}_t^{\ours} = A_t + \gamma \cdot r_{\text{conf}}(x_t).
\end{equation}
Since we already compute teacher confidence for the top-$K$ candidates, we extend to a multi-sample estimator. Letting $A_t^{(k)} = 1 + \log \pi_\theta(x_t^{(k)} | \mathbf{x}_{<t}) - \log \pi_T(x_t^{(k)} | \mathbf{x}_{<t})$, the \ourstopk{} loss weights each candidate by its student probability:
\begin{equation} \label{eq:topk_loss}
    \mathcal{L}_t^{\ourstopk} = \sum_{k=1}^{K} \pi_\theta(x_t^{(k)} | \mathbf{x}_{<t}) \cdot \left[ A_t^{(k)} + \gamma \cdot r_{\text{conf}}(x_t^{(k)}) \right].
\end{equation}

\subsection{Efficient Computation} \label{sec:entropy_trigger}

\textbf{Entropy trigger.} Computing the confidence reward for all $L$ positions with $K$ candidates each would be prohibitively expensive. We observe that at positions where the \emph{student} has low generation entropy, the top-1 token dominates the probability mass, and the remaining candidates carry negligible weight in the policy gradient. We trigger the confidence reward only at positions $\mathcal{S} = \{t : \Ent(\pi_\theta(\cdot | \mathbf{x}_{<t})) > \tau \}$; elsewhere, only the standard reverse-KL loss is applied. This typically selects $|\mathcal{S}| \approx 0.15L \text{--} 0.25L$ positions early in training, stabilizing around $0.10L$ as the student distribution sharpens, reducing computational overhead by $4\text{--}7$ times.

The triggering uses \emph{student} entropy rather than teacher entropy, avoiding circularity: teacher entropy at $t$ is exactly what degrades under SFD, so triggering on it would disable the reward precisely where it is most needed. This selection is near-lossless (Appendix~\ref{app:obs:trigger}).

\textbf{Tree attention.} Naively evaluating $K$ candidates at each $t \in \mathcal{S}$ requires $K \cdot |\mathcal{S}|$ teacher forward prefill for student's prefix. We instead construct an extended sequence (main sequence $\mathbf{x}$ followed by all candidates as branches) with a tree-structured mask $\mathbf{M}$ ~\citep{cai2024medusasimplellminference,li2025eaglespeculativesamplingrequires,li2025eagle3scalinginferenceacceleration} : main branch tokens attend causally; each candidate $x_t^{(k)}$ attends to $\mathbf{x}_{\leq t-1}$ but not to other candidates (see Figure~\ref{fig:tree_attn} for an example). In practice, GPU memory limits the total sequence length, so candidates are processed in $N$ segments rather than all at once. Appendix~\ref{app:tree_attn} gives the full cost analysis.

\begin{algorithm}[t]
\small
\caption{\ours{}: Lookahead Group Reward}
\label{alg:tgr}
\begin{algorithmic}[1]
\setlength{\itemsep}{1pt}
\REQUIRE Student $\pi_\theta$, Teacher $\pi_T$, entropy threshold $\tau$, top-$K$, reward weight $\gamma$
\FOR{each training step}
    \STATE Sample prompt $\mathbf{c}$ from dataset
    \STATE Generate $\mathbf{x} = (x_1, \ldots, x_L) \sim \pi_\theta(\cdot | \mathbf{c})$ \hfill \textit{// Student rollout}
    \STATE Compute student logits and entropy $\Ent_t$ for all positions
    \STATE Identify high-entropy set $\mathcal{S} = \{t : \Ent_t > \tau\}$
    \STATE Extract top-$K$ candidate tokens at each $t \in \mathcal{S}$
    \STATE Construct tree-attention input: main sequence + all candidates
    \STATE Construct multi-branch tree mask $\mathbf{M}$
    \STATE Run teacher forward pass with tree mask $\mathbf{M}$ \hfill \textit{// Single pass for all positions}
    \STATE Extract $r_{\text{raw}}^{(k)}$ for all candidates; compute $r_{\text{conf}}$ via Eq.~\eqref{eq:group_norm}
    \FOR{each position $t = 1, \ldots, L$}
        \IF{$t \in \mathcal{S}$}
            \STATE $\mathcal{L}_t \gets \mathcal{L}_t^{\ours}$ via Eq.~\eqref{eq:topk_loss} \hfill \textit{// Local + lookahead}
        \ELSE
            \STATE $\mathcal{L}_t \gets \log \frac{\pi_\theta(x_t | \mathbf{x}_{<t})}{\pi_T(x_t | \mathbf{x}_{<t})}$ \hfill \textit{// Standard R-KL}
        \ENDIF
    \ENDFOR
    \STATE Update $\theta$ with $\nabla_\theta \sum_t \mathcal{L}_t$
\ENDFOR
\end{algorithmic}
\end{algorithm}
\vspace{-1em}

\section{Experiments} \label{sec:exp}

\subsection{Experimental Setup} \label{sec:setup}
Full training and evaluation details are provided in Appendix~\ref{app:training_details}.

\textbf{Models.} We evaluate two teacher--student configurations sharing the same vocabulary, each with a single teacher:
(1) \textbf{1.5B student:} DeepSeek-R1-Distill-Qwen-1.5B~\citep{Guo_2025} as the student, trained separately for math (teacher: SkyWork-OR1-Math-7B~\citep{he2025skyworkopenreasoner1}) and code (teacher: DeepCoder-14B~\citep{balog2017deepcoderlearningwriteprograms}).
(2) \textbf{7B student:} DeepSeek-R1-Distill-Qwen-7B as the student, with DeepSeek-R1-Distill-Qwen-32B as the teacher.

\textbf{Benchmarks.} We evaluate our method on several rigorous datasets. For mathematical reasoning, we utilize AIME-24~\citep{aime24}, AIME--25~\citep{aime25}, and AIME-26~\citep{aime26}, as well as the February sessions of HMMT-25 and HMMT-26~\citep{hmmt}. For evaluation of code generation, we employ the LiveCodeBench v6 suite~\citep{jain2024livecodebenchholisticcontaminationfree}.

\textbf{Baselines.} We compare \ours{} against: (1) \textbf{GRPO}~\citep{Guo_2025}: Group Relative Policy Optimization with outcome-level reward; (2) \textbf{OPD}: on-policy distillation with standard reverse-KL, plus a top-$K$ multi-sample variant (OPD topk); (3) \textbf{JSD}~\citep{agarwal2024onpolicydistillationlanguagemodels}: Jensen--Shannon divergence distillation; and (4) \textbf{REOPOLD}~\citep{ko2026scalingreasoningefficientlyrelaxed}: Relaxed On-Policy Distillation, which stabilizes RKL training via mixture-based reward clipping (to handle heavy-tailed negative rewards) and entropy-guided token-level dynamic sampling (to filter near-zero reward tokens).

\subsection{Main Results} \label{sec:main_results}

\begin{table}[t]
\caption{Comparison of distillation methods on mathematical reasoning and code benchmarks (mean@8 / pass@8). $\Delta$ denotes the absolute difference between our best method and OPD (9k).}
\label{tab:main}
\centering
\resizebox{\textwidth}{!}{%
\begin{tabular}{lccccccc}
\toprule
 & \multicolumn{5}{c}{\textbf{Math}} & \multicolumn{1}{c}{\textbf{Code}} & \\
\cmidrule(lr){2-6} \cmidrule(lr){7-7}
\textbf{Model} & \textbf{AIME-24} & \textbf{AIME-25} & \textbf{AIME-26} & \textbf{HMMT-25} & \textbf{HMMT-26} & \textbf{LCB\_v6} & \textbf{AVG.} \\
\midrule
\multicolumn{8}{c}{\textbf{SkyWork-OR1-Math-7B / DeepCoder-14B} $\rightarrow$ \textbf{DeepSeek-R1-Distill-Qwen-1.5B}} \\
\midrule
SkyWork-OR1-Math-7B & 67.50/83.33 & 52.08/66.67 & 62.08/80.00 & 32.22/55.17 & 36.74/54.55 & -- & -- \\
DeepCoder-14B       & --          & --          & --          & --          & --          & 76.87/86.34 & -- \\
\midrule
Base Model          & 24.17/50.00 & 21.67/33.33 & 17.92/46.67 & 11.67/23.33 & 12.50/27.27 & 10.13/31.06 & 16.34/35.28 \\
\quad + GRPO        & 38.33/66.67 & 30.83/43.33 & 28.75/50.00 & 16.74/20.69 & 20.46/27.27 & 24.56/40.31 & 26.61/41.38 \\
\quad + OPD (9k)    & \underline{45.83}/80.00 & 33.75/46.67 & 32.08/56.67 & 20.42/36.67 & 22.73/27.27 & \underline{34.58}/50.44 & 31.57/49.62 \\
\quad + OPD (topk)  & \underline{45.83}/70.00 & 31.67/43.33 & \underline{33.75}/66.67 & \underline{20.50}/27.59 & 21.59/33.33 & 26.04/37.89 & 29.90/46.47 \\
\quad + JSD         & 38.75/66.67 & \underline{34.17}/50.00 & 32.50/50.00 & 18.75/23.33 & 21.21/33.33 & 29.02/46.92 & 29.07/45.04 \\
\quad + REOPOLD     & 41.67/60.00 & 33.33/46.67 & 30.00/46.67 & 19.17/26.67 & 22.35/27.27 & 33.59/49.56 & 30.02/42.81 \\
\midrule
\quad + \ours{}        & \textbf{46.67}/73.33 & \textbf{35.42}/50.00 & \textbf{34.17}/63.33 & \textbf{22.08}/36.67 & \underline{23.86}/37.50 & \textbf{36.89}/53.97 & \textbf{33.18}/52.47 \\
\quad + \ours{} (topk) & 44.17/80.00 & \textbf{35.42}/56.67 & \underline{33.75}/63.33 & \textbf{22.08}/46.67 & \textbf{24.62}/34.38 & 33.43/50.00 & \underline{32.25}/55.18 \\
\midrule
$\Delta$ & \textcolor{green!50!black}{+0.84} & \textcolor{green!50!black}{+1.67} & \textcolor{green!50!black}{+2.09} & \textcolor{green!50!black}{+1.66} & \textcolor{green!50!black}{+1.89} & \textcolor{green!50!black}{+2.31} & \textcolor{green!50!black}{+1.61} \\
\midrule\midrule
\multicolumn{8}{c}{\textbf{DeepSeek-R1-Distill-Qwen-32B} $\rightarrow$ \textbf{DeepSeek-R1-Distill-Qwen-7B}} \\
\midrule
DeepSeek-R1-Distill-Qwen-32B & 80.83/93.33 & 50.83/73.33 & 70.00/83.33 & 50.83/70.00 & 43.94/63.64 & 68.34/82.16 & -- \\
\midrule
Base Model          & 52.08/80.00 & 43.33/63.33 & 42.92/63.33 & 24.17/43.33 & 28.03/39.39 & 41.96/61.23 & 38.75/58.44 \\
\quad + GRPO        & 67.50/83.33 & 52.08/66.67 & 62.08/80.00 & 32.22/55.17 & 36.74/54.55 & 50.67/69.33 & 50.22/68.17 \\
\quad + OPD (9k)    & \textbf{62.08}/76.67 & 45.42/63.33 & 52.50/76.67 & 27.50/50.00 & \underline{32.95}/45.45 & 54.24/72.01 & 45.78/64.02 \\
\quad + OPD (topk)  & 57.50/80.00 & 45.00/66.67 & 52.08/70.00 & 29.58/56.67 & 30.30/42.42 & 53.36/71.81 & 44.64/64.60 \\
\quad + JSD         & 60.83/83.33 & 45.83/56.67 & 51.25/66.67 & 25.83/43.33 & 31.82/42.42 & 50.39/68.01 & 44.33/60.07 \\
\quad + REOPOLD     & 57.08/83.33 & \underline{47.50}/70.00 & \textbf{54.17}/80.00 & \underline{30.00}/56.67 & \underline{32.95}/46.88 & 54.41/72.47 & 46.02/68.23 \\
\midrule
\quad + \ours{}        & \underline{61.67}/83.33 & \textbf{49.58}/73.33 & \underline{53.75}/80.00 & \textbf{31.67}/60.00 & \textbf{34.85}/51.52 & \textbf{58.59}/74.45 & \textbf{48.35}/70.44 \\
\quad + \ours{} (topk) & \underline{61.67}/80.00 & 47.08/66.67 & \underline{53.75}/76.67 & 27.92/53.33 & 30.68/36.36 & \underline{56.61}/73.34 & \underline{46.29}/64.40 \\
\midrule
$\Delta$ & \textcolor{red!70!black}{-0.41} & \textcolor{green!50!black}{+4.16} & \textcolor{green!50!black}{+1.25} & \textcolor{green!50!black}{+4.17} & \textcolor{green!50!black}{+1.90} & \textcolor{green!50!black}{+4.35} & \textcolor{green!50!black}{+2.57} \\
\bottomrule
\end{tabular}%
}
\vspace{-1em}
\end{table}

Table~\ref{tab:main} presents the main comparison across both teacher--student configurations.

\textbf{\ours{} delivers superior aggregate performance across benchmarks.} In the 1.5B student setting, \ours{} achieves an average of 33.18\% mean@8, which is an increase of 1.61\% absolute over OPD 9k. These gains are consistent across all six evaluation sets, with the largest improvements occurring on AIME-26 and LCB\_v6. For the larger 7B student model, \ours{} widens the lead to an average improvement of 2.57\%. While it performs competitively but slightly behind the baseline on AIME-244, \ours{} yields significant gains on the remaining five benchmarks, particularly on AIME-25 and LCB\_v6 where improvements exceed 4.1\%.

\textbf{The advantage is more pronounced at larger scale.} In the 7B student setting, \ours{} achieves 48.35\% average mean@8, a +2.57\% improvement over OPD 9k. The gains are particularly large on AIME-25 (+4.16\%), HMMT-25 (+4.17\%), and LCB\_v6 (+4.35\%), suggesting that \ours{} is especially effective when the student has sufficient capacity to leverage the improved supervision signal. 

\textbf{\ours{} vs.\ \ours{} (topk).} The top-$K$ extension provides complementary but not strictly additive gains. In the 1.5B setting, \ours{} (topk) achieves notably higher pass@8 on several benchmarks (e.g., 56.67\% vs.\ 50.00\% on AIME-25), suggesting that the multi-sample estimator helps with diversity. In the 7B setting, \ours{} without top-$K$ generally outperforms, indicating that the single-sample estimator with confidence reward is sufficient at larger scale.

\textbf{Stabilization alone does not address SFD.} JSD and REOPOLD achieve comparable or lower performance than OPD, confirming that the core issue is not the choice of divergence or optimization instability, but the degradation of teacher supervision quality on student-generated contexts, a problem that demands a structurally different solution.

\subsection{\ours{} Addresses Supervision Fidelity Decay} \label{sec:fidelity_analysis}
\vspace{-0.5em}

\begin{figure}[t]
    \centering
    \includegraphics[width=\textwidth]{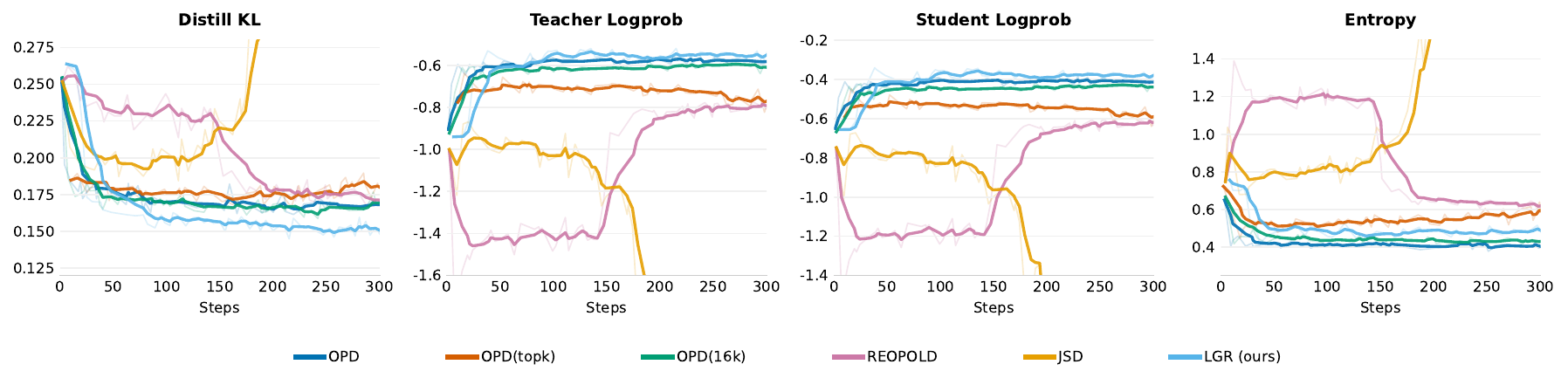}
    \caption{\textbf{Training dynamics across distillation methods.} Comparison of metrics over training steps. \ours{} maintains higher teacher log-probability and more stable entropy.}
    \label{fig:training_dynamics}
    \vspace{-1em}
\end{figure}

\textbf{Training dynamics.} Figure~\ref{fig:training_dynamics} tracks  distill KL, teacher log-probability, student log-probability and entropy during training. Compared to OPD variants, \ours{} maintains higher teacher log-probability throughout training, indicating that the student's generated contexts remain closer to the teacher's in-distribution manifold. The entropy curves show that \ours{} maintains more stable generation diversity rather than collapsing into the mode-sharpening behavior predicted by Proposition~\ref{prop:drift}.

\begin{wraptable}{r}{0.55\textwidth}
\centering
\caption{\textbf{OPD vs.\ \ours{} across max generation lengths.} Mean@8/Pass@8 on AIME benchmarks (1.5B student).}
\label{tab:length_compare}
\resizebox{\linewidth}{!}{%
\begin{tabular}{llccc}
\toprule
\textbf{Max Len} & \textbf{Method} & \textbf{AIME-24} & \textbf{AIME-25} & \textbf{AIME-26} \\
\midrule
\multirow{2}{*}{3k}  & OPD  & 42.92/76.67 & 33.33/50.00 & 32.92/60.00 \\
                      & \ours{} & 41.25/66.67 & 32.92/40.00 & 27.50/43.33 \\
\midrule
\multirow{2}{*}{9k}  & OPD  & 45.83/\textbf{80.00} & 33.75/46.67 & 32.08/56.67 \\
                      & \ours{} & \textbf{46.67}/73.33 & 35.42/50.00 & 34.17/63.33 \\
\midrule
\multirow{2}{*}{16k} & OPD  & 44.17/73.33 & 33.75/50.00 & 34.17/63.33 \\
                      & \ours{} & 46.25/76.67 & 36.25/53.33 & 33.33/63.33 \\
\midrule
\multirow{2}{*}{39k} & OPD  & 43.33/66.67 & 32.08/53.33 & 30.00/53.33 \\
                      & \ours{} & 46.25/73.33 & \textbf{36.75}/\textbf{60.00} & \textbf{34.92}/\textbf{66.67} \\
\bottomrule
\end{tabular}%
}
\end{wraptable}

\textbf{\ours{}'s advantage grows with maximum generation length.} Table~\ref{tab:length_compare} compares OPD and \ours{} at different maximum generation lengths (3k, 9k, 16k, 39k). At short lengths (3k), where SFD is minimal, \ours{} provides no advantage because the teacher's local supervision is already reliable. As the maximum length increases, \ours{} increasingly outperforms OPD. At 39k, \ours{} achieves the strongest gains across all three AIME benchmarks (e.g., 36.75 vs.\ 32.08 on AIME-25, 34.92 vs.\ 30.00 on AIME-26). This pattern is precisely what our SFD analysis predicts: the confidence reward becomes increasingly valuable as the teacher's local supervision degrades at longer positions.

\begin{wrapfigure}[15]{r}[0pt]{0.33\textwidth}
    \vspace{-2em}
    \centering
    \includegraphics[width=0.3\textwidth]{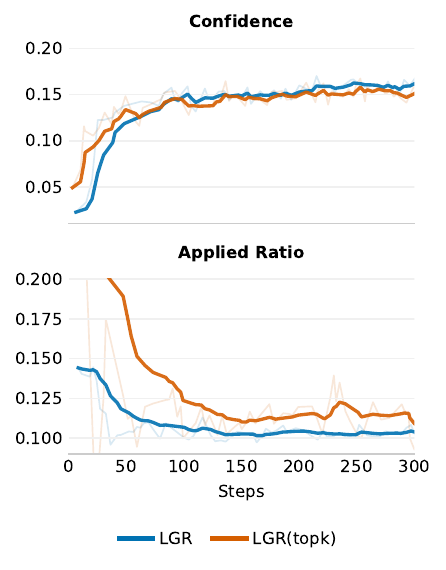}
    \vspace{-1em}
    \caption{\textbf{\ours{} confidence reward dynamics.} \textit{Top:} The average confidence reward increases over training. \textit{Bottom:} The applied ratio stabilizes around 10\%.}
    \label{fig:confidence_dynamics}
    \vspace{-2em}
\end{wrapfigure}

\textbf{Confidence reward dynamics.} Figure~\ref{fig:confidence_dynamics} tracks the \ours{} confidence reward over training. The average confidence reward increases over training (top panel), indicating that the student progressively learns to generate tokens that lead to higher teacher confidence at the next position. The applied ratio (bottom panel) stabilizes around 10\%, showing that the entropy-triggered activation identifies a consistent fraction of high-entropy decision points. This stable activation ratio confirms that the entropy trigger effectively identifies positions where lookahead guidance is most needed, without requiring manual tuning over the course of training.

\subsection{Ablation Studies} \label{sec:ablation}

\textbf{Confidence metric comparison.} We compare four candidate confidence metrics for the group-normalized reward (Table~\ref{tab:ablation}a). Max-$p$ consistently outperforms the alternatives, achieving 46.67 on AIME-24 compared to 43.75 for Sampled-$p$, 42.08 for entropy, and 37.92 for PPL. The strong advantage of Max-$p$ aligns with Proposition~\ref{prop:maxp_proximity}: it directly measures the proximity to the teacher's in-distribution regime, while PPL and entropy are noisier proxies that can be inflated by irrelevant low-probability tokens.

\begin{table}[t]
\centering
\caption{\textbf{Ablation studies} on AIME benchmarks (Mean@8/Pass@8, 1.5B student). \textit{Left:} Confidence metric comparison. \textit{Right:} Sensitivity to confidence weight $\gamma$.}
\label{tab:ablation}
\begin{minipage}[t]{0.48\linewidth}
\centering
\textbf{(a) Confidence metric}\\[4pt]
\label{tab:confidence_metric}
\footnotesize
\begin{tabular}{lccc}
\toprule
\textbf{Metric} & \textbf{AIME-24} & \textbf{AIME-25} & \textbf{AIME-26} \\
\midrule
PPL         & 37.92/70.00 & 26.67/43.33 & 26.67/53.33 \\
Entropy     & 42.08/73.33 & 29.58/46.67 & 25.83/53.33 \\
Sampled-$p$ & 43.75/70.00 & 29.17/43.33 & 32.50/50.00 \\
Max-$p$     & \textbf{46.67/73.33} & \textbf{35.42/50.00} & \textbf{34.17/63.33} \\
\bottomrule
\end{tabular}
\end{minipage}
\hfill
\begin{minipage}[t]{0.48\linewidth}
\centering
\textbf{(b) Confidence weight $\gamma$}\\[4pt]
\label{tab:gamma}
\footnotesize
\begin{tabular}{lccc}
\toprule
$\gamma$ & \textbf{AIME-24} & \textbf{AIME-25} & \textbf{AIME-26} \\
\midrule
0.1   & 43.33/73.33 & 31.67/46.67 & 32.08/56.67 \\
1.0   & \textbf{46.67}/73.33 & \textbf{35.42/50.00} & \textbf{34.17/63.33} \\
1.5   & 44.58/\textbf{76.67} & 34.17/50.00 & 33.75/63.33 \\
10.0  & 37.08/70.00 & 30.00/40.00 & 28.33/50.00 \\
\bottomrule
\end{tabular}
\end{minipage}
\vspace{-1em}
\end{table}

\textbf{Sensitivity to confidence weight $\gamma$.} Table~\ref{tab:ablation}b illustrates how the choice of $\gamma$ affects the distillation outcomes. When $\gamma$ is as low as 0.1, the influence of the lookahead signal is negligible, and the performance remains close to the standard OPD baseline. The optimal results are achieved in the range of 1.0 to 1.5, where the confidence reward effectively guides the student toward states that maintain high supervision quality. In contrast, setting $\gamma$ to 10.0 causes a significant decline in accuracy. In this high weight regime, the optimization becomes susceptible to reward hacking, as the student model tends to generate repetitive token sequences that artificially inflate teacher confidence but lack the logical substance required to solve the task correctly.
\section{Related Work} \label{sec:related}


\textbf{On-Policy Distillation.}
Knowledge distillation~\citep{hinton2015distillingknowledgeneuralnetwork, ko2025distillm2contrastiveapproachboosts,DBLP:journals/corr/KimR16a,ye2026blackboxonpolicydistillationlarge,hübotter2026reinforcementlearningselfdistillation,kim2021comparingkullbackleiblerdivergencemean} transfers capabilities from teacher to student. For LLMs, training on teacher-generated (off-policy) data creates a distribution mismatch at inference time~\citep{bengio2015scheduledsamplingsequenceprediction, he2021exposurebiasversusselfrecovery,kim2026distillationlargelanguagemodels}. MiniLLM~\citep{gu2026minillmonpolicydistillationlarge} addresses this by minimizing reverse-KL on student-generated rollouts to avoid the mode-averaging problem of forward-KL; GKD~\citep{agarwal2024onpolicydistillationlanguagemodels} further demonstrates that on-policy generated data consistently outperforms off-policy training. These works establish OPD as the standard paradigm for reasoning model compression. Subsequent work has explored several directions. On \emph{training stability}: REOPOLD~\citep{ko2026scalingreasoningefficientlyrelaxed} stabilizes reverse-KL training via mixture-based reward clipping and entropy-guided token-level dynamic sampling. On \emph{objective design}: G-OPD~\citep{yang2026learningteachergeneralizedonpolicy} formalizes OPD as KL-constrained RL and proposes reward extrapolation to learn beyond the teacher ceiling; TSD-KD~\citep{kim2026explainwordsimprovingreasoning} applies KL loss selectively on high-entropy tokens where the student is genuinely uncertain, reducing noise from low-uncertainty positions. 

Two concurrent works are most closely related. Revisiting OPD~\citep{fu2026revisit} identify that OPD becomes unreliable on student generations and propose teacher top-$K$ support matching. Rethinking OPD~\citep{li2026rethinkingonpolicydistillationlarge} study OPD phenomenology and find that reward quality degrades with trajectory depth---consistent with our SFD analysis. Our work differs in two respects: (1) we provide a formal analysis of \emph{why} and \emph{where} supervision degrades---characterizing SFD as a position-dependent functional that worsens monotonically along student trajectories and establishes a reasoning length ceiling; and (2) we propose a one-step lookahead reward remedy that directly optimizes teacher confidence at the future position, orthogonal to divergence modification and applicable on top of OPD objective.

\section{Conclusion} \label{sec:conclusion}


We proposed Lookahead Group Reward, which augments standard reverse-KL distillation with a group-normalized confidence reward that directly optimizes the teacher's supervision capability functional. Unlike masking or truncation strategies that passively avoid weak-supervision regions, \ours{} actively steers the student toward trajectories where the teacher maintains high supervision fidelity. The entropy-triggered tree-attention mechanism makes this approach computationally practical. Achieving a 1000$\times$ speedup compared to a naïve implementation.


\textbf{Limitations and future work.} First, \ours{} requires white-box access to the teacher model's logits, which may not always be available. Second, the confidence reward relies on an implicit assumption that the teacher's \emph{relative ranking} of candidate tokens remains informative even when its absolute predictions are unreliable. While our group normalization design and empirical results support this assumption, it may weaken for extremely out-of-distribution student trajectories or poorly calibrated teachers. Third, the current design uses a fixed number of candidates $K$ and a static entropy threshold $\tau$; dynamically adjusting these based on training progress could improve both efficiency and effectiveness. Finally, extending the confidence reward framework to multi-modal reasoning settings represents a promising direction.

\newpage
\bibliographystyle{nips}
\bibliography{references}

\appendix

\section*{Technical Appendices and Supplementary Material}

\noindent\textbf{Table of Contents}
\begin{enumerate}[label=\Alph*.,leftmargin=*]
    \item \hyperref[app:training_details]{Training and Evaluation Details} \dotfill \pageref{app:training_details}
    \item \hyperref[app:additional_exp]{Additional Experimental Results} \dotfill \pageref{app:additional_exp}
    \begin{enumerate}[label=\Alph{enumi}.\arabic*.,leftmargin=*]
        \item \hyperref[app:exp:unnorm]{Effect of Renormalization in Top-$K$ Training} \dotfill \pageref{app:exp:unnorm}
        \item \hyperref[app:exp:temp]{Effect of Rollout Temperature} \dotfill \pageref{app:exp:temp}
        \item \hyperref[app:exp:topk_vis]{Per-Token Top-$K$ Logit Visualization} \dotfill \pageref{app:exp:topk_vis}
        \item \hyperref[app:exp:future_kl]{Future-KL with GAE Weighting} \dotfill \pageref{app:exp:future_kl}
        \item \hyperref[app:exp:distil_kl]{Comparison of KL Divergence Objectives} \dotfill \pageref{app:exp:distil_kl}
    \end{enumerate}
    \item \hyperref[app:sigmoid]{Sigmoid Fit of the SFD Curve} \dotfill \pageref{app:sigmoid}
    \item \hyperref[app:obs:trigger]{Entropy-Triggered Activation is Near-Lossless} \dotfill \pageref{app:obs:trigger}
    \item \hyperref[app:tree_attn]{Tree Attention: Cost Analysis and Practical Segmentation} \dotfill \pageref{app:tree_attn}
    \item \hyperref[app:proofs]{Proofs of Theoretical Results} \dotfill \pageref{app:proofs}
    \begin{enumerate}[label=\Alph{enumi}.\arabic*.,leftmargin=*]
        \item \hyperref[app:proof:snr]{Proof of Proposition~\ref{prop:snr} (Teacher Signal Vanishing under SFD)} \dotfill \pageref{app:proof:snr}
        \item \hyperref[app:proof:drift]{Proof Sketch of Proposition~\ref{prop:drift} (Self-Reinforcing Drift under Reverse-KL)} \dotfill \pageref{app:proof:drift}
        \item \hyperref[app:proof:one_step]{Proof of Proposition~\ref{prop:one_step} (One-Step-Ahead Discriminability)} \dotfill \pageref{app:proof:one_step}
        \item \hyperref[app:proof:maxp]{Proof of Proposition~\ref{prop:maxp_proximity} (Max-$p$ as Relative Drift Indicator)} \dotfill \pageref{app:proof:maxp}
        \item \hyperref[app:proof:groupnorm]{Group Normalization: Design Rationale and Formal Properties} \dotfill \pageref{app:proof:groupnorm}
    \end{enumerate}
\end{enumerate}


\section{Training and Evaluation Details} \label{app:training_details}

\textbf{Training framework.} All models are trained on 4 nodes of 8$\times$H20 GPUs (32 GPUs total) with the SLIME framework~\citep{slime_github}, using SGLang~\citep{zheng2024sglangefficientexecutionstructured} as the inference backend. We made two adaptations to SGLang: (1) we extended it to support tree-structured attention masks, enabling the single-pass multi-candidate teacher evaluation described in Section~\ref{sec:entropy_trigger}; (2) we patched SGLang to return temperature-scaled logits for the prefill portion, which was required for experiments exploring non-unit rollout temperatures (see below).

\textbf{Training data.} For math training we use the Polaris dataset; for code training we use the DeepScaler code dataset. This applies uniformly across all student model sizes.

\textbf{Rollout temperature.} All experiments use rollout temperature~$= 1.0$ for both student and teacher. We also evaluated two non-unit configurations: (1) $t_{\mathrm{student}}=1, t_{\mathrm{teacher}}=0.2$ (sharpened teacher distribution); (2) $t_{\mathrm{student}}=0.8, t_{\mathrm{teacher}}=0.8$ (symmetric low temperature). Both configurations degraded performance. We therefore fix temperature~$= 1.0$ across all methods and settings.

\textbf{Baseline-specific configurations.}
\begin{itemize}[leftmargin=*,topsep=2pt,itemsep=2pt]
    \item \textbf{GRPO:} Rather than training GRPO from scratch, we report results from publicly available RL-trained checkpoints to ensure a fair and reproducible comparison. For the \textbf{1.5B student} setting, we use DeepScaler-1.5B, which was obtained by applying GRPO to DeepSeek-R1-Distill-Qwen-1.5B on a large-scale math dataset. For the \textbf{7B student} setting, we use SkyWork-OR1-Math-7B, which was obtained by applying GRPO to DeepSeek-R1-Distill-Qwen-7B. Both checkpoints share the same base model as the corresponding distillation experiments, making the comparison directly controlled for initialization.
    \item \textbf{OPD (topk) and \ourstopk{}:} The per-token top-$K$ candidate set is the union of the student's top-10 and teacher's top-10 tokens, renormalized to a valid probability distribution. Renormalization is essential: without it the joint candidate set has inconsistent probability mass and training diverges. The union construction also guarantees that teacher-preferred tokens are always included even when the student assigns them low probability.
    \item \textbf{REOPOLD:} We follow the original training configuration but set staleness~$= 1$ (fully on-policy). The original paper uses staleness~$= 4$; we found this setting to be unstable in our experiments, likely because the larger policy lag interacts poorly with the entropy-triggered reward clipping mechanism.
    \item \textbf{JSD:} We use mixture coefficient $\beta = 0.2$ (i.e., $\pi_{\mathrm{mix}} = 0.2\,\pi_T + 0.8\,\pi_\theta$).
\end{itemize}

\begin{table}[h]
\centering
\caption{Training hyperparameters for the two student configurations.}
\label{tab:training_details}
\begin{tabular}{lcc}
\toprule
\textbf{Hyperparameter} & \textbf{1.5B Student} & \textbf{7B Student} \\
\midrule
Optimizer               & AdamW                 & AdamW               \\
Adam $\beta_1$, $\beta_2$ & 0.9,~0.98           & 0.9,~0.98           \\
Weight decay            & 0.1                   & 0.1                 \\
Learning rate           & $1 \times 10^{-5}$    & $3 \times 10^{-6}$  \\
Batch size              & 128                   & 128                 \\
Warmup steps            & 0                     & 0                   \\
Training steps          & 300                   & 300                 \\
Rollout temperature     & 1.0                   & 1.0                 \\
Rollout samples per prompt & 1                 & 1                   \\
\midrule
\multicolumn{3}{l}{\textit{\ours{}-specific}} \\
Top-$K$ candidates ($K$) & 8                   & 8                   \\
Entropy threshold ($\tau$) & 0.2              & 0.2                 \\
Confidence reward weight ($\gamma$) & 1.0    & 1.0                 \\
Tree attention segments ($N$) & 8            & 8                   \\
\bottomrule
\end{tabular}
\end{table}

\textbf{Evaluation.} All models are evaluated with a maximum generation length of 39k tokens for math benchmarks and 36k tokens for code benchmarks (code prompts are longer, leaving less budget for generation), at temperature~$= 1$. We report mean@8 and pass@8 across 8 responses per problem.

\textbf{On the choice of student models.} Our main experiments use DeepSeek-R1-Distill models as students (1.5B and 7B), which are initialized from a base model via supervised fine-tuning on reasoning traces. We also experimented with applying OPD to Qwen3 reasoning models (specifically, distilling from Qwen3-32B to Qwen3-1.7B), but found that training is highly unstable: performance initially improves but then degrades progressively rather than converging---a pattern visible in Figures~\ref{fig:unnorm} and~\ref{fig:temp}, where even the best-configured runs show initial gains followed by gradual degradation. We attribute this to the nature of the Qwen3 reasoning models, which are the result of extensive integrated multi-domain RL fusion applied to a broad mixture of tasks---yielding a well-balanced student that has already converged on the teacher's distribution across diverse contexts. When such a model is used as a student in further OPD, the training signal is dominated by the small residual distribution mismatch rather than systematic supervision failures, making the optimization landscape highly sensitive and difficult to stabilize. Note that our results differ significantly from those reported by Thinking Machine Lab on similar model families: their experiments use a base model or an SFT-from-base model as the student, rather than a fully trained reasoning model, which explains the more stable training dynamics they observe.

In contrast, DeepSeek-R1-Distill models are derived from a base model through reasoning-focused supervised learning, without the breadth of integrated multi-domain RL fusion. This leaves meaningful room for OPD to provide corrective supervision and makes the SFD phenomenon clearly observable---as the student has not already adapted to the teacher's distribution on student-generated prefixes. Using this model family therefore provides a cleaner testbed for diagnosing and addressing SFD, and the instability observed on Qwen3 further underscores the importance of student model selection in OPD experiments.


\section{Additional Experimental Results} \label{app:additional_exp}

\subsection{Homogeneous vs.\ Heterogeneous Teacher--Student Pairs} \label{app:exp:homo_hetero}

In on-policy distillation, the teacher and student can either share the same base model (\emph{homogeneous}) or originate from different model families (\emph{heterogeneous})~\citep{lin-etal-2020-autoregressive,patiño2025_unlocking_on_policy_distillation_for_any_model_family}. In the homogeneous setting, the teacher is obtained by further training the student base model via RLVR, so the two models share identical tokenization and pretraining priors---the student's on-policy distribution is close to the teacher's from the outset. In the heterogeneous setting (used throughout this paper), the student and teacher come from different base models, introducing a structural distribution gap that is present even before any distillation training begins.

Figures~\ref{fig:homo_scatter} and~\ref{fig:hetero_scatter} compare the joint distribution of per-token student log-probability and teacher log-probability at training steps 0, 40, and 80 for the two settings. In the homogeneous case, the scatter plots show a tight correlation between student and teacher log-probabilities throughout training: the student's on-policy tokens remain well within the teacher's in-distribution region, and this alignment persists even as training progresses. In the heterogeneous case, the scatter is substantially wider at all steps, with the student frequently visiting token positions where the teacher assigns low probability---precisely the regime in which supervision fidelity degrades. This structural distribution gap is why the SFD phenomenon is clearly observable in the heterogeneous setting, and why our experiments adopt this configuration as the primary testbed.

\begin{figure}[h]
    \centering
    \subfigure[Step 0]{\includegraphics[width=0.31\textwidth]{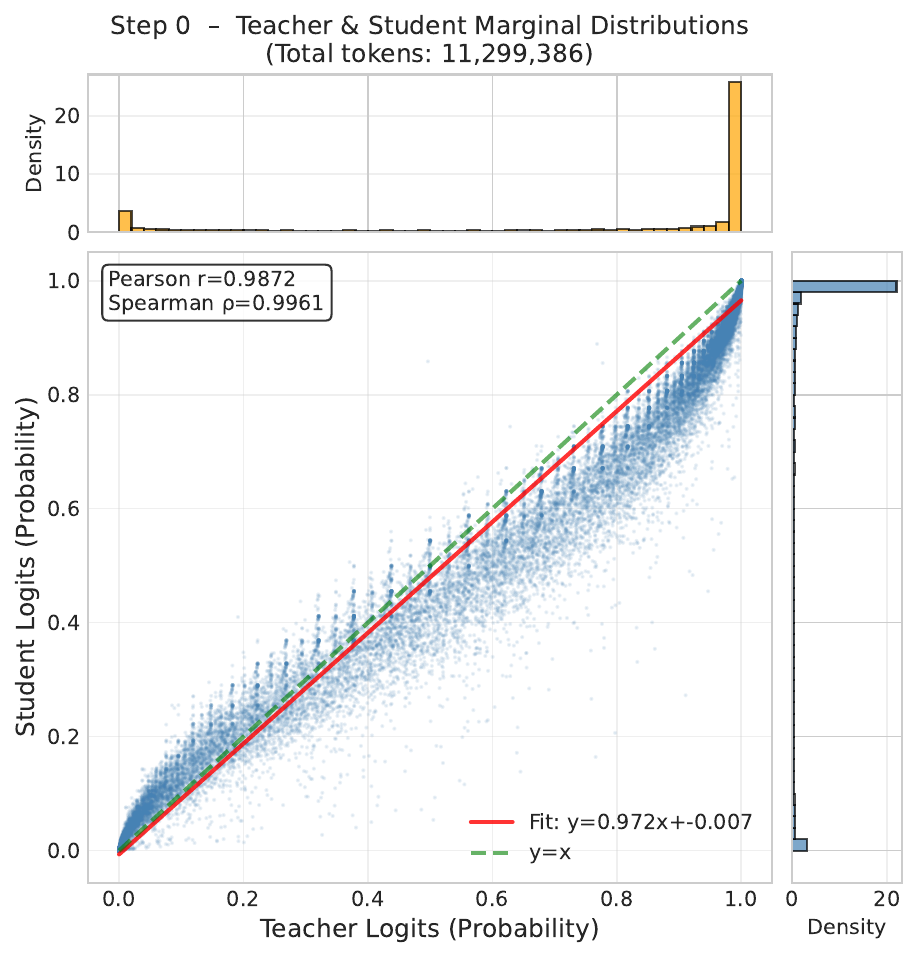}}
    \hfill
    \subfigure[Step 40]{\includegraphics[width=0.31\textwidth]{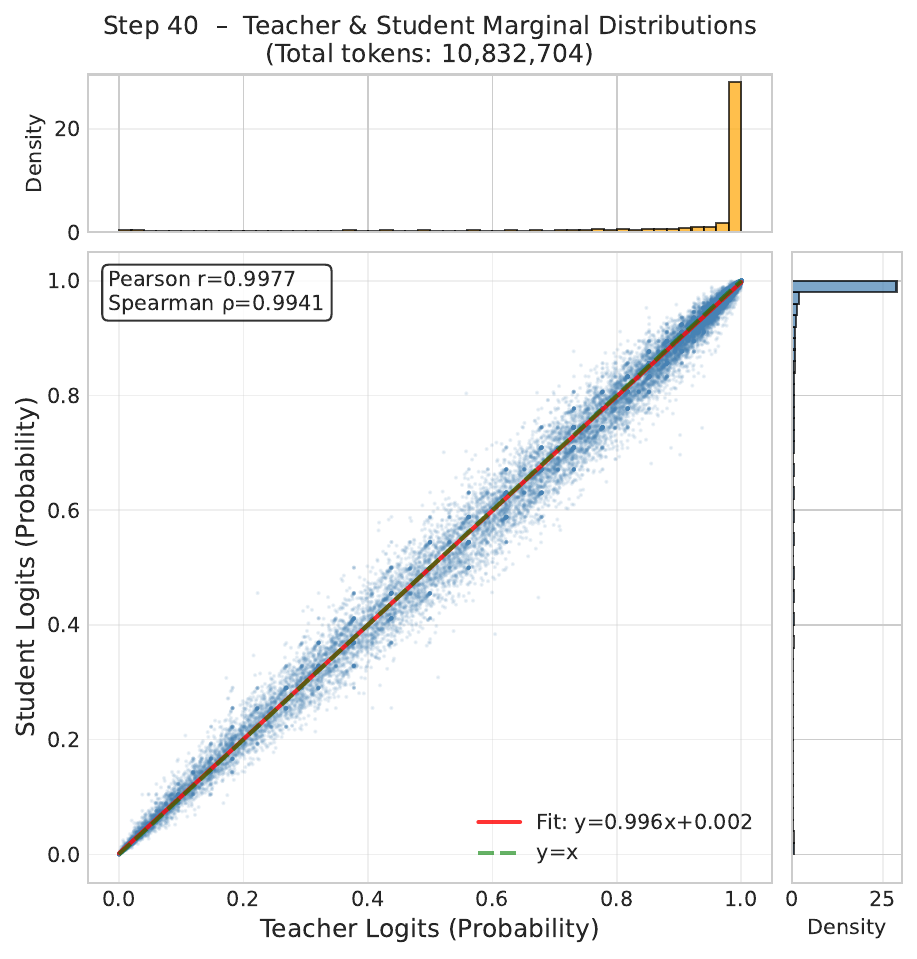}}
    \hfill
    \subfigure[Step 80]{\includegraphics[width=0.31\textwidth]{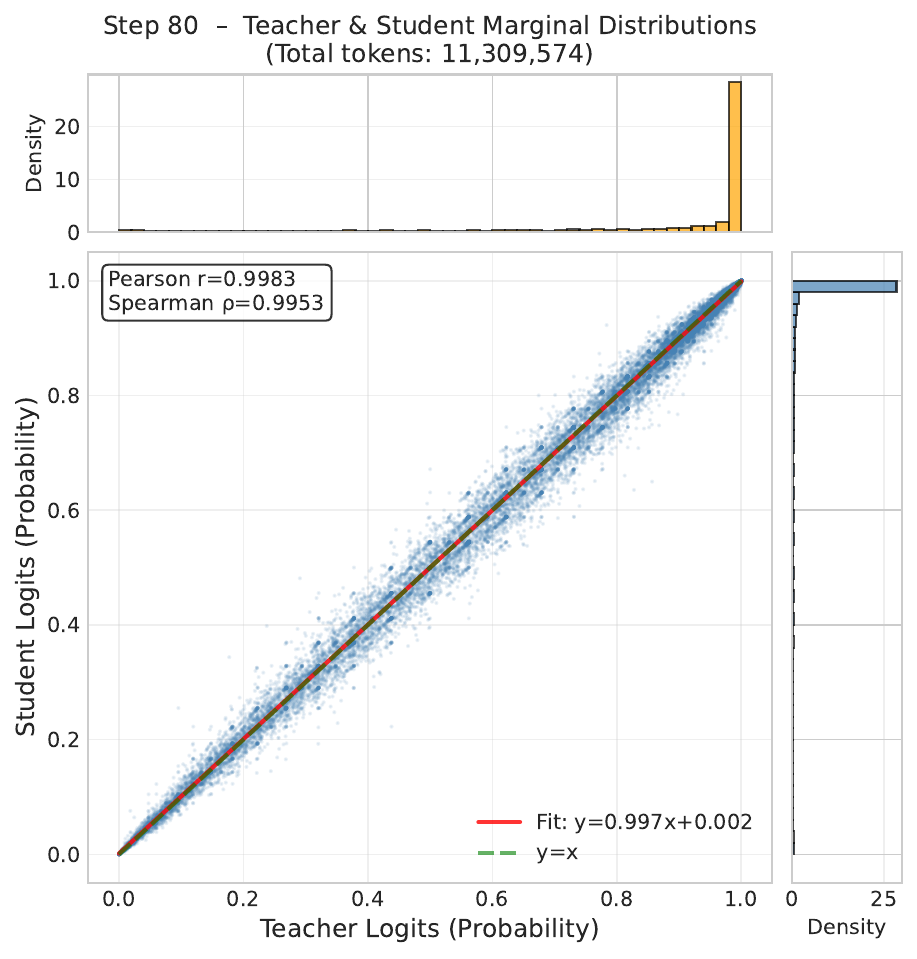}}
    \caption{Homogeneous teacher--student pair: joint distribution of per-token student log-probability vs.\ teacher log-probability at training steps 0, 40, and 80. The teacher is obtained by further RLVR training from the same student base model, so the two models share pretraining priors and the student tokens remain well within the teacher's in-distribution region throughout training.}
    \label{fig:homo_scatter}
\end{figure}

\begin{figure}[h]
    \centering
    \subfigure[Step 0]{\includegraphics[width=0.31\textwidth]{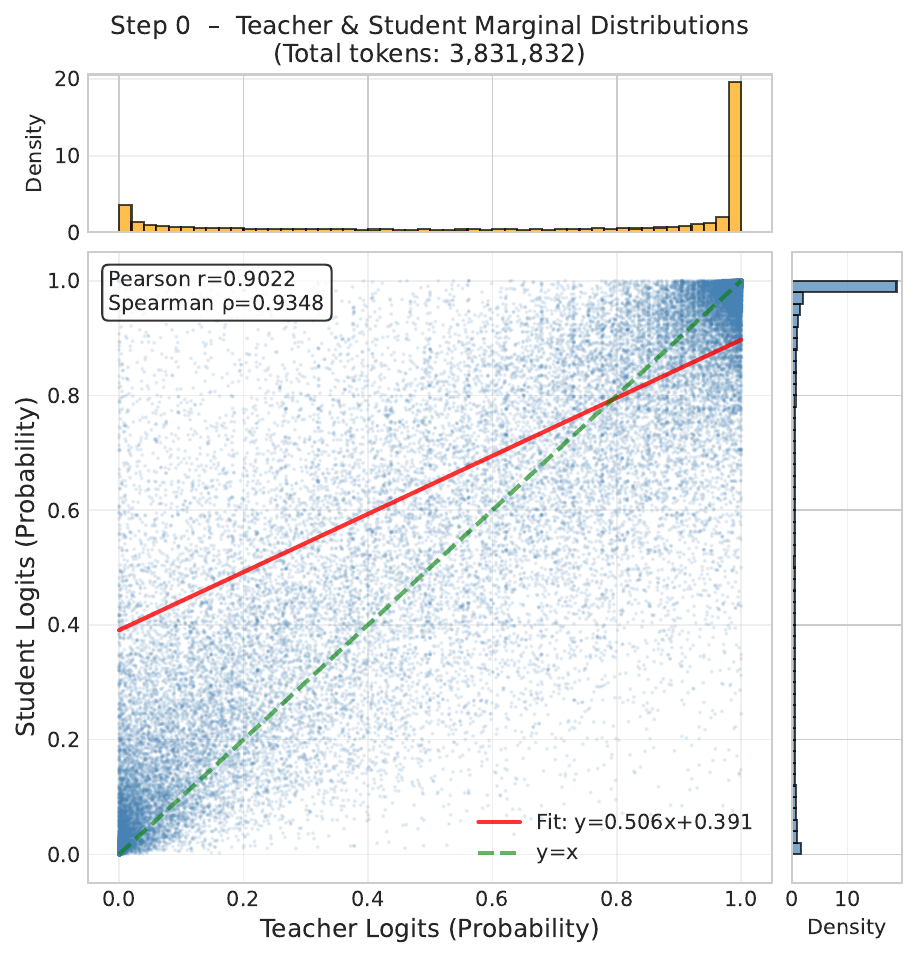}}
    \hfill
    \subfigure[Step 40]{\includegraphics[width=0.31\textwidth]{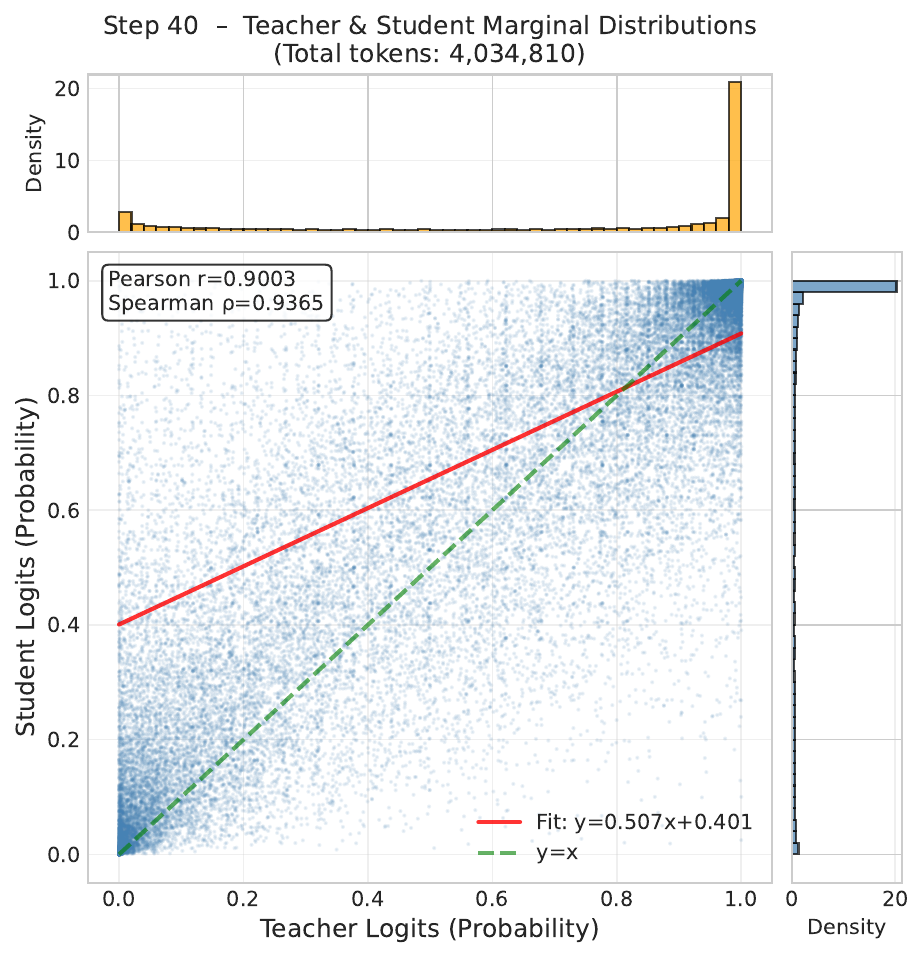}}
    \hfill
    \subfigure[Step 80]{\includegraphics[width=0.31\textwidth]{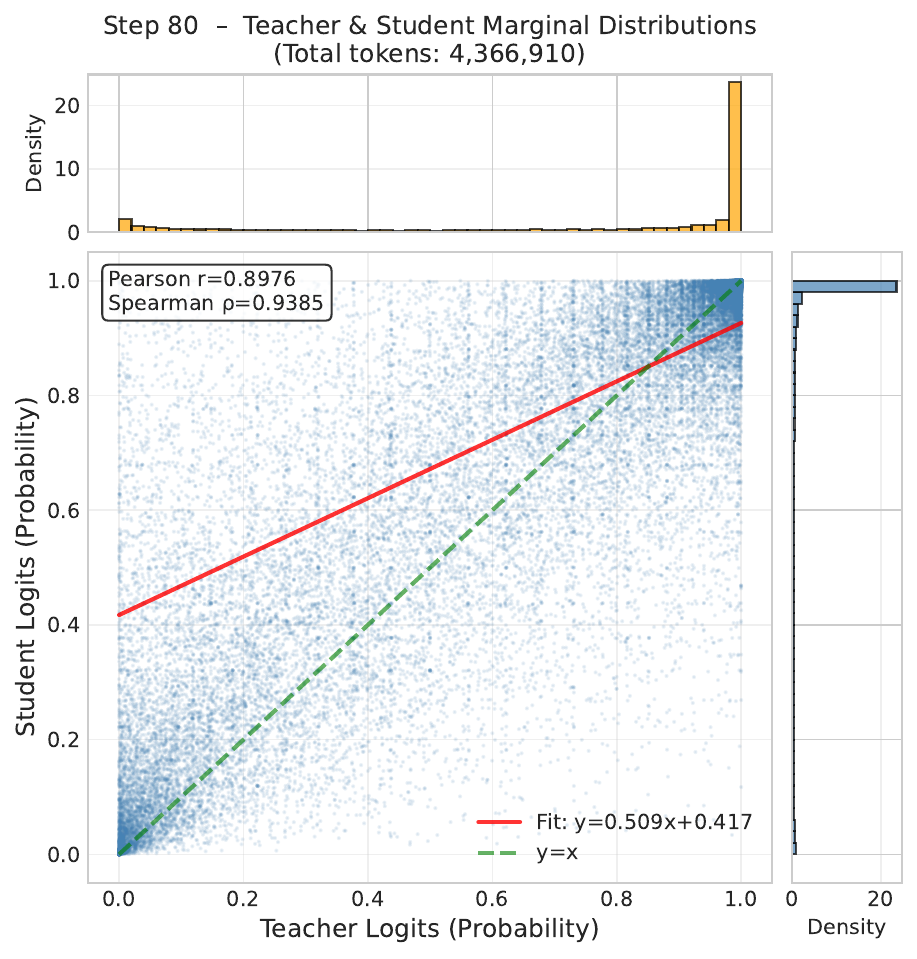}}
    \caption{Heterogeneous teacher--student pair: joint distribution of per-token student log-probability vs.\ teacher log-probability at training steps 0, 40, and 80. The teacher and student originate from different base model families. The scatter is substantially wider than the homogeneous case, indicating that the student frequently generates tokens in low-probability regions of the teacher's distribution---the primary driver of supervision fidelity decay.}
    \label{fig:hetero_scatter}
\end{figure}

\subsection{Effect of Renormalization in Top-$K$ Training} \label{app:exp:unnorm}

In OPD (topk) and \ourstopk{}, each token's distribution is restricted to the union of the student's top-10 and teacher's top-10 candidates, and the resulting truncated distribution is renormalized to sum to one before computing the KL loss. Here we ablate the necessity of this renormalization step by training the same models without it---i.e., computing the KL directly over the raw unnormalized top-$K$ union probabilities.

Figure~\ref{fig:unnorm} compares OPD (topk) with and without renormalization on the Qwen3-1.7B student configuration across AIME benchmarks. We make three observations.

\textbf{Unnormalized training shows higher teacher and student log-probabilities.} Counterintuitively, models trained without renormalization exhibit \emph{higher} teacher and student log-probabilities throughout training. We attribute this to an artifact of the truncation: without renormalization, the top-$K$ probabilities do not sum to one, which effectively inflates the absolute probability of each candidate token. During rollout sampling this inflation makes it \emph{easier} to draw low-probability tokens, distorting the on-policy distribution and masking the true quality of the supervision signal.

\textbf{Performance improves with larger top-$K$ but never matches renormalized training.} Among the unnormalized variants, increasing the total top-$K$ size (top-5 $\to$ top-10) improves performance, suggesting that a wider candidate set provides more signal. However, even at the same total top-$K$ count, unnormalized training consistently underperforms its renormalized counterpart. This is because without renormalization, all probability mass assigned to tokens \emph{outside} the top-$K$ union is completely invisible to the KL loss---the model can shift mass to out-of-vocabulary positions without incurring any penalty.

\textbf{Renormalization closes an optimization loophole.} Without renormalization, the model can learn a degenerate strategy: push probability mass outside the top-$K$ support (escaping KL penalty entirely), while simultaneously lowering the absolute probabilities of tokens within the top-$K$ set, making the measured KL appear small. With renormalization, any mass that leaks outside the top-$K$ union is implicitly redistributed back by the normalization step---the normalized probabilities within the top-$K$ set rise whenever out-of-support mass increases, removing any incentive for this exploit and forcing the student to genuinely match the teacher on the retained candidates.

\begin{figure}[h]
    \centering
    \includegraphics[width=\textwidth]{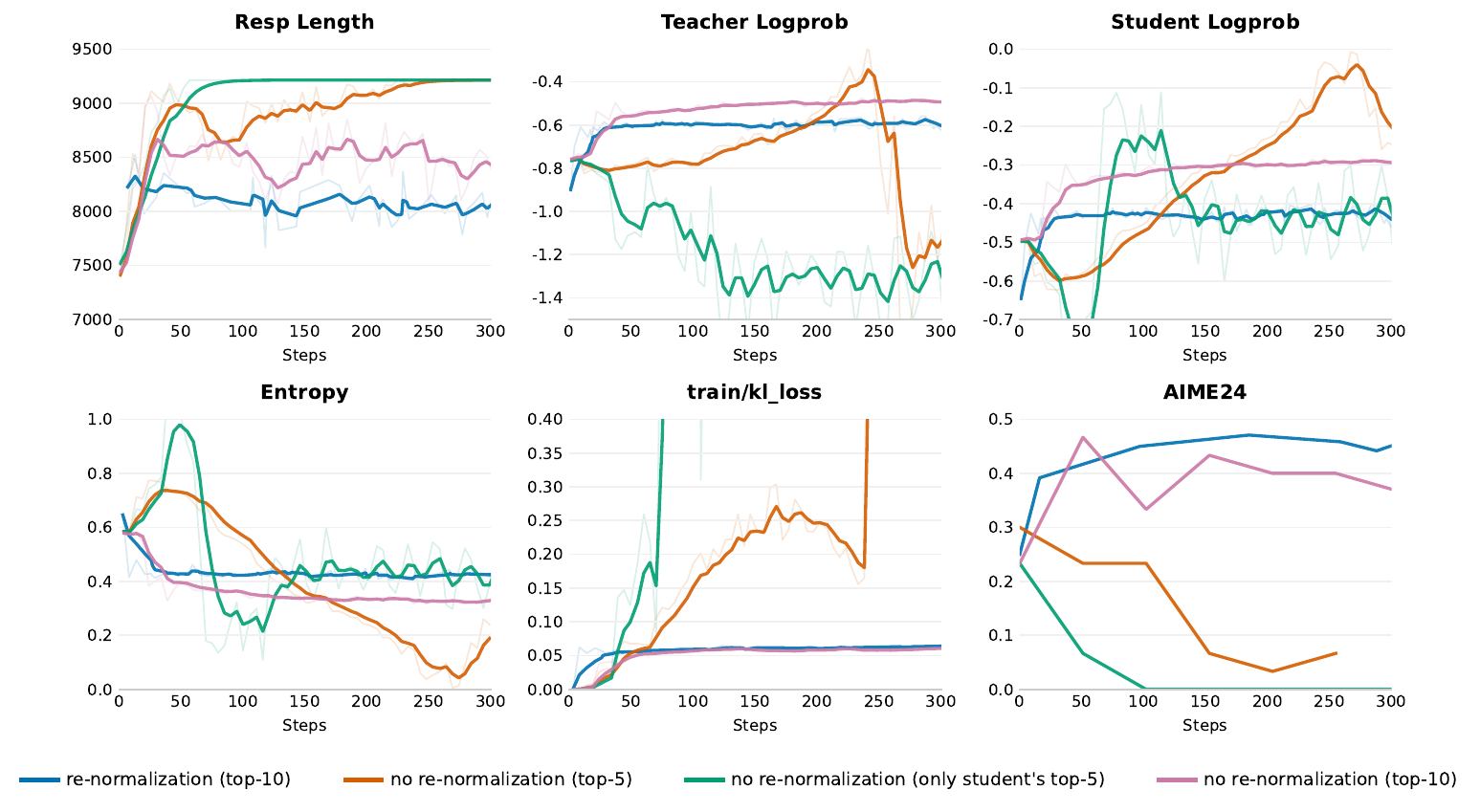}
    \caption{Effect of renormalization in top-$K$ distillation training (Qwen3-1.7B student). Training curves and AIME-24 performance for OPD (topk) with renormalization vs.\ three unnormalized variants (top-5, top-10, student-only top-5). Without renormalization, training is unstable and final performance degrades substantially.}
    \label{fig:unnorm}
\end{figure}

\subsection{Effect of Rollout Temperature} \label{app:exp:temp}

All main experiments fix the rollout temperature to $1.0$ for both the student and the teacher. We additionally evaluate two non-unit configurations: (1) $t_{\mathrm{student}}=1, t_{\mathrm{teacher}}=0.2$ (sharpened teacher distribution); (2) $t_{\mathrm{student}}=0.8, t_{\mathrm{teacher}}=0.8$ (symmetric low temperature).

\textbf{$t_{\mathrm{student}}=1, t_{\mathrm{teacher}}=0.2$.} A lower teacher temperature concentrates teacher probability mass sharply on its top tokens. Intuitively this sharpens the supervision signal, but it also increases the mismatch between the temperature-conditioned teacher logits and the student's on-policy distribution, causing training instability.

\textbf{$t_{\mathrm{student}}=0.8, t_{\mathrm{teacher}}=0.8$.} Lowering both temperatures symmetrically reduces generation diversity and dilutes the discriminative signal between high- and low-confidence token choices, weakening the confidence reward.

Figure~\ref{fig:temp} shows training curves and final evaluation scores under both settings compared to the default temperature-$1.0$ baseline. Both non-unit configurations degrade performance, confirming that the symmetric unit temperature is the best operating point and that the SGLang prefill temperature-scaling patch (Appendix~\ref{app:training_details}) is not needed in the final system.

\begin{figure}[h]
    \centering
    \includegraphics[width=\textwidth]{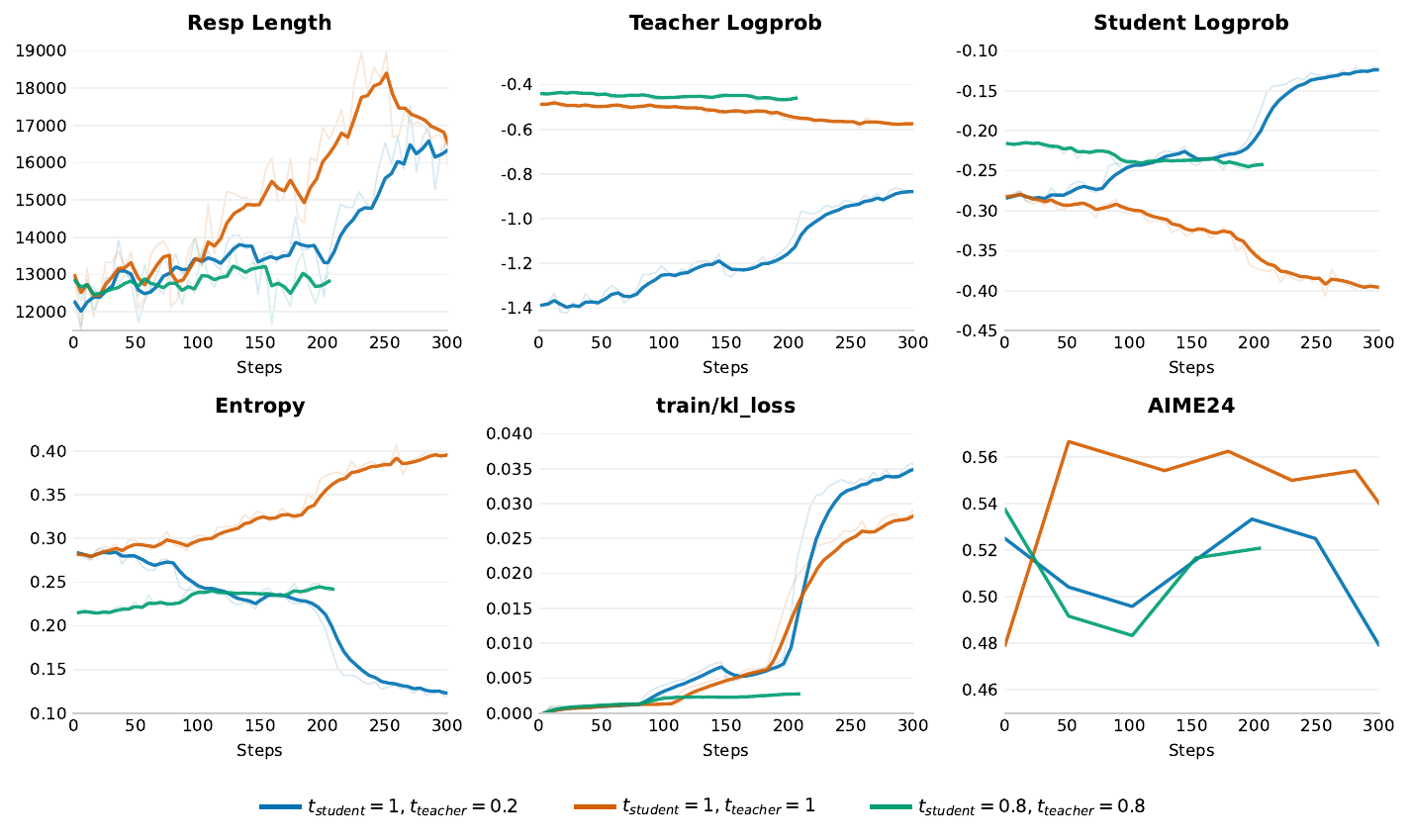}
    \caption{Effect of rollout temperature on training dynamics and final AIME-24 performance (Qwen3-1.7B student). Conditions: $t_{\mathrm{student}}=1, t_{\mathrm{teacher}}=1$ (default), $t_{\mathrm{student}}=1, t_{\mathrm{teacher}}=0.2$ (sharpened teacher), and $t_{\mathrm{student}}=0.8, t_{\mathrm{teacher}}=0.8$ (symmetric low temperature). Both non-unit configurations degrade performance, confirming that symmetric unit temperature is the optimal operating point.}
    \label{fig:temp}
\end{figure}

\subsection{Per-Token Top-$K$ Logit Visualization} \label{app:exp:topk_vis}

To understand \emph{why} per-token RKL remains stubbornly high at certain positions even after an OPD gradient step, we visualize the full top-$K$ logit distribution of the student at individual tokens, comparing the distributions before and after each optimization step.

\textbf{High-RKL tokens have dispersed student top-$K$ logits.} Figure~\ref{fig:topk_vis} shows representative tokens from a student rollout late in training. For tokens where the RKL is large and remains large after the update step, the student's top-$K$ probability mass is spread relatively uniformly across many candidates---the student is genuinely uncertain, and no single token dominates. The teacher, by contrast, concentrates mass sharply on one or two tokens. The gradient step reduces the KL slightly but cannot collapse the student distribution in a single step given the flat landscape.

\textbf{Low-RKL tokens have logits concentrated at top-1.} Tokens that achieve low RKL exhibit the opposite pattern: the student distribution is already sharply peaked at the same top-1 token as the teacher. These tokens contribute near-zero loss and near-zero gradient.

\textbf{Persistent high-RKL late in training.} As training progresses, the proportion of low-RKL (peaked) tokens grows, but a residual population of high-RKL (dispersed) tokens persists and proves resistant to further optimization. Crucially, in the standard per-token average loss, the large and growing majority of low-RKL tokens effectively act as a low-magnitude denominator that dilutes the gradient contribution of the remaining high-RKL tokens. Each high-RKL token's gradient is upweighted in absolute terms but its relative influence in the batch average decreases as low-RKL tokens accumulate---creating a natural but undesirable gradient imbalance.

\begin{figure}[h]
    \centering
    \subfigure{\includegraphics[width=0.49\textwidth]{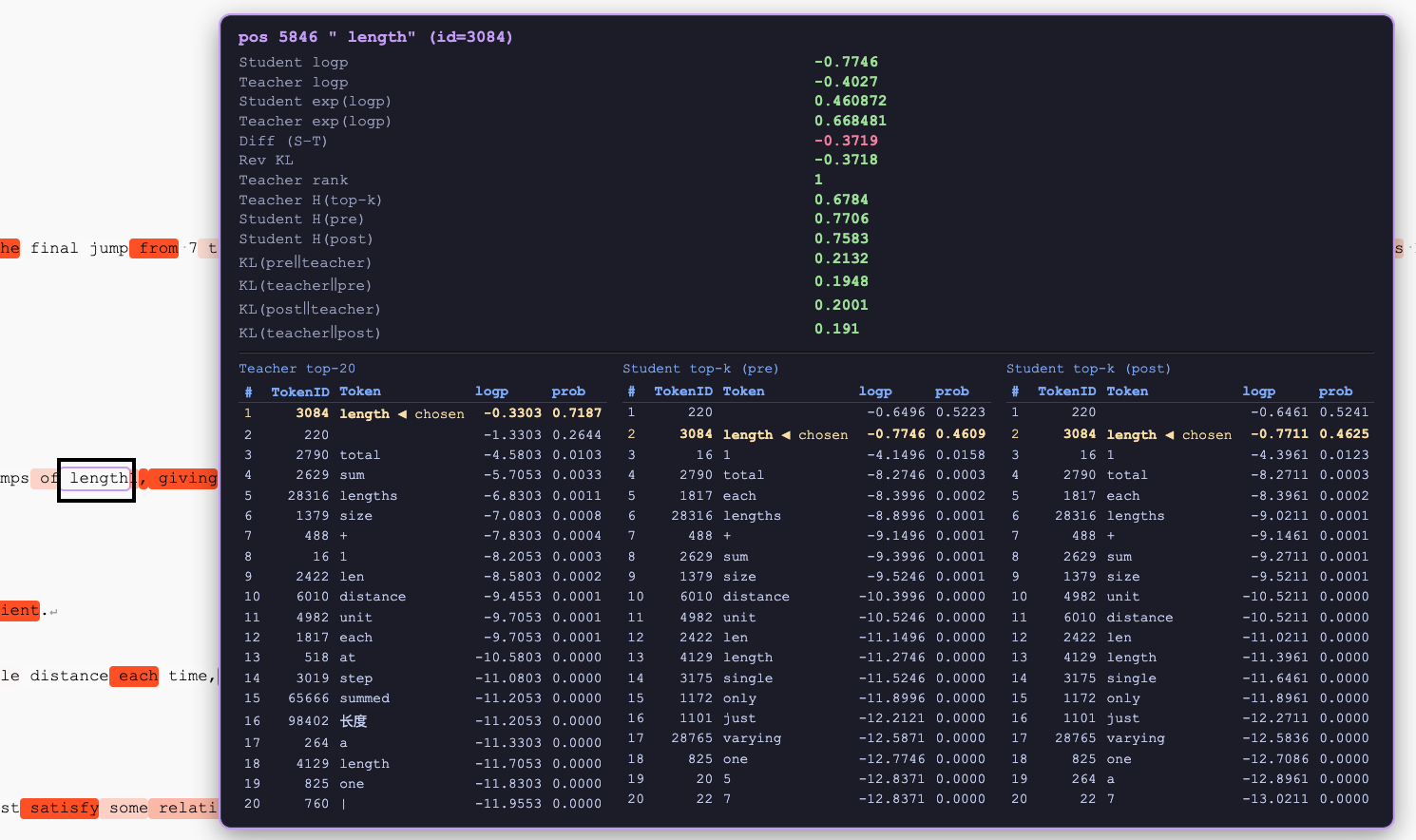}}
    \hfill
    \subfigure{\includegraphics[width=0.49\textwidth]{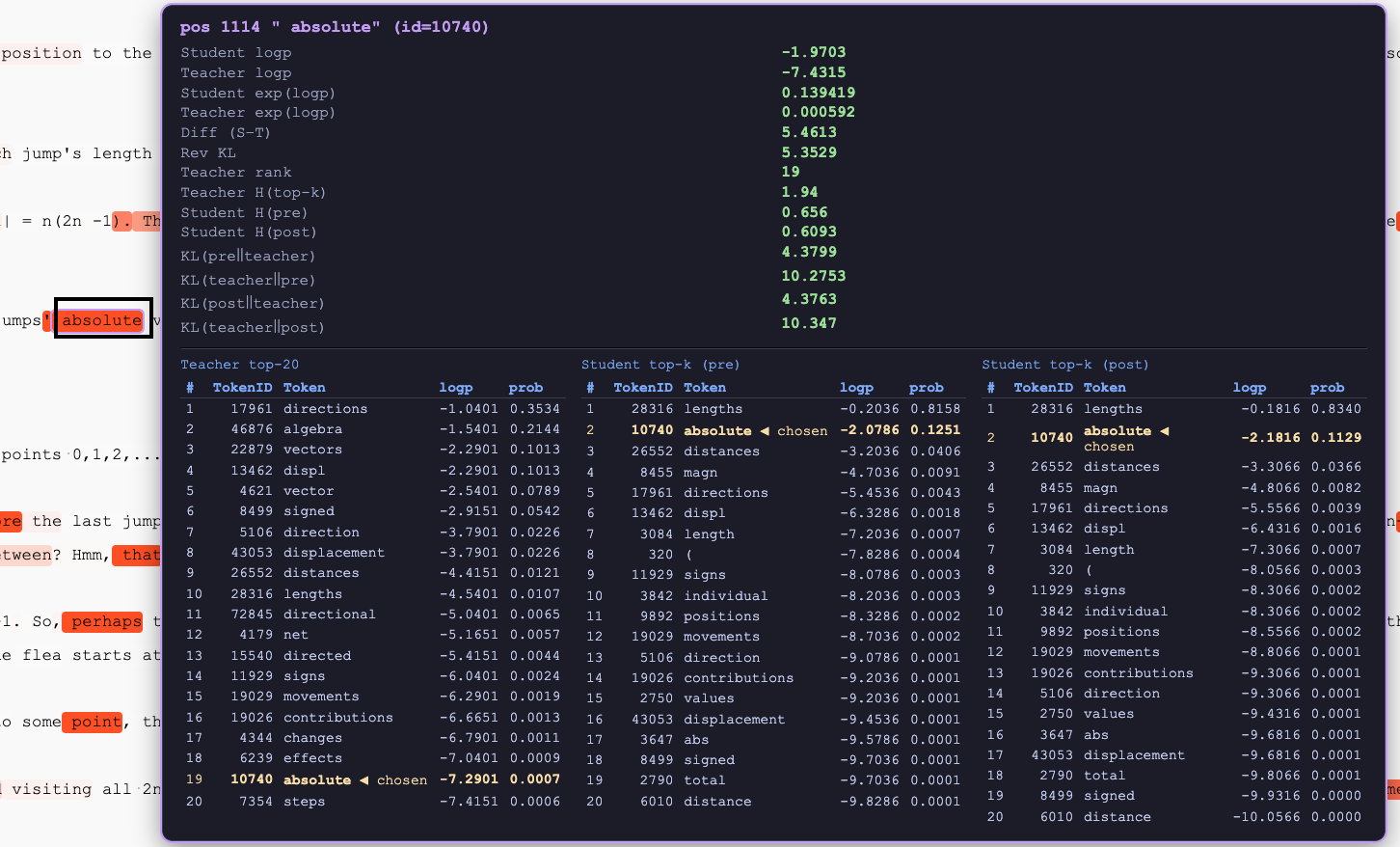}}
    \vspace{0.5em}
    \subfigure{\includegraphics[width=0.49\textwidth]{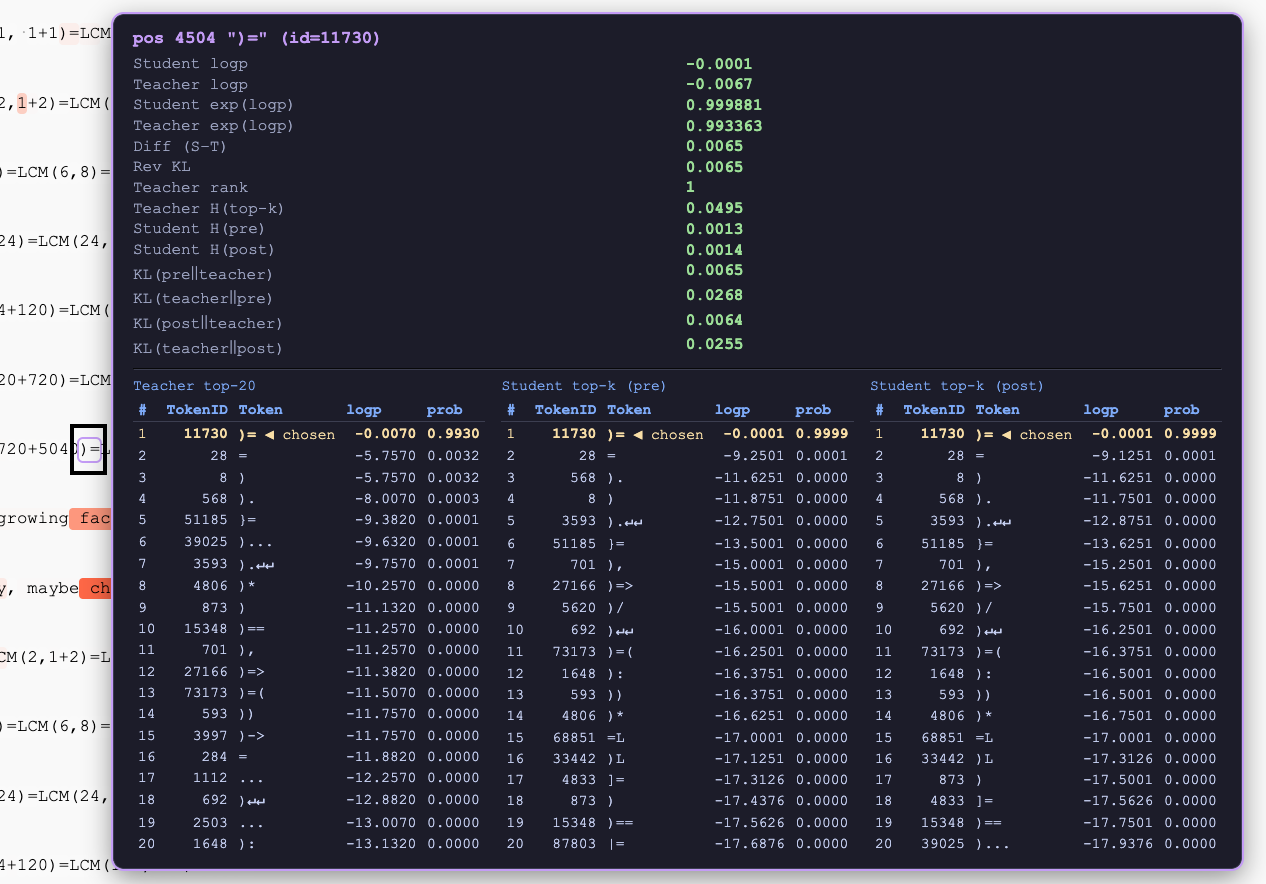}}
    \hfill
    \subfigure{\includegraphics[width=0.49\textwidth]{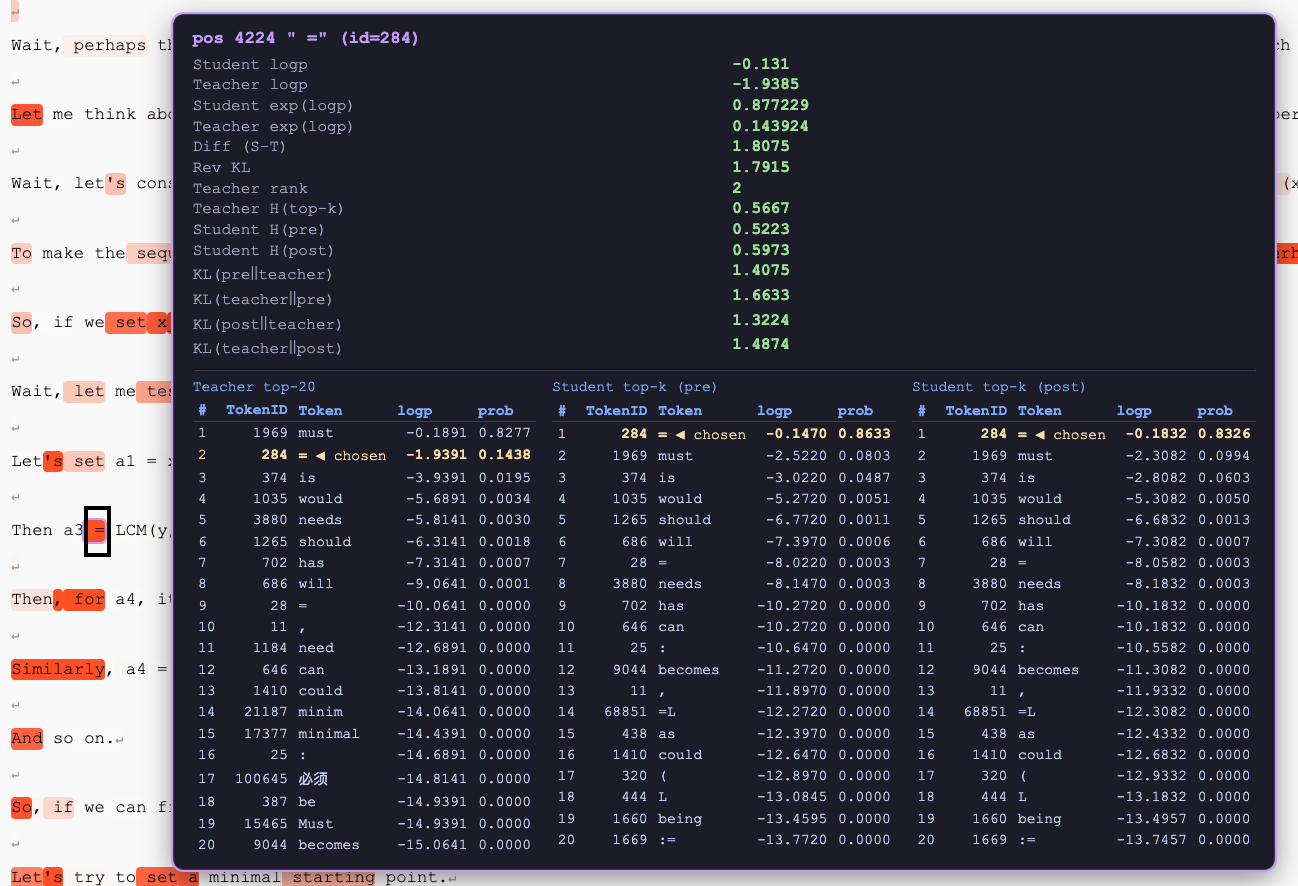}}
    \caption{Per-token top-$K$ student logit distributions before and after an OPD update step for four representative tokens from a student rollout late in training (DeepSeek-R1-Distill-Qwen-1.5B student). Each panel shows the token context, student top-$K$ logits and probabilities, teacher top-$K$ logits, and the per-token KL divergence before and after the gradient step. High-RKL tokens exhibit dispersed student probability mass across many candidates, while the teacher concentrates sharply on one or two tokens; the gradient step reduces the KL only marginally, consistent with the flat optimization landscape described in Section~\ref{app:exp:topk_vis}.}
    \label{fig:topk_vis}
\end{figure}

\subsection{Future-KL with GAE Weighting} \label{app:exp:future_kl}

Motivated by the credit assignment literature in RL, we explore an alternative to the local per-token RKL loss: rather than penalizing the student only for the divergence at position $t$, we compute a \emph{future-KL} signal by aggregating the RKL over all future tokens $t' > t$ and weighting them with a generalized advantage estimate (GAE, $\lambda$-return) decay:
\begin{equation}
    \mathcal{L}_{\mathrm{future\text{-}KL}}^{(t)} = \sum_{t'=t}^{T} (\rho\lambda)^{t'-t}\,\mathrm{RKL}\!\left(\pi_\theta(\cdot \mid x_{<t'})\,\|\,\pi_T(\cdot \mid x_{<t'})\right),
\end{equation}
where $\rho$ is a discount factor and $\lambda$ is the GAE trace-decay parameter. \textbf{Note that this $\lambda$ is the GAE eligibility-trace coefficient and is distinct from the confidence reward weight $\gamma$ in \ours{}; the two hyperparameters play entirely different roles.} Intuitively, this encourages the student to make choices at position $t$ that lead to lower KL \emph{throughout} the trajectory, not just locally---which could in principle counteract the self-reinforcing drift described in Proposition~\ref{prop:drift}.

Figure~\ref{fig:future_kl} compares future-KL distillation against standard OPD under a sweep of $(\rho, \lambda)$ pairs. Despite the appealing intuition, future-KL consistently fails to improve over standard per-token RKL. We hypothesize that the difficulty lies in the credit assignment itself: in long reasoning chains, the future-KL signal at early positions is dominated by the noise accumulated over hundreds of subsequent tokens, making the gradient at position $t$ effectively uninformative about the local decision quality. This suggests that the right way to address long-horizon supervision degradation is through the one-step-ahead lookahead used in \ours{}, rather than through discounted future aggregation.

\begin{figure}[h]
    \centering
    \includegraphics[width=\textwidth]{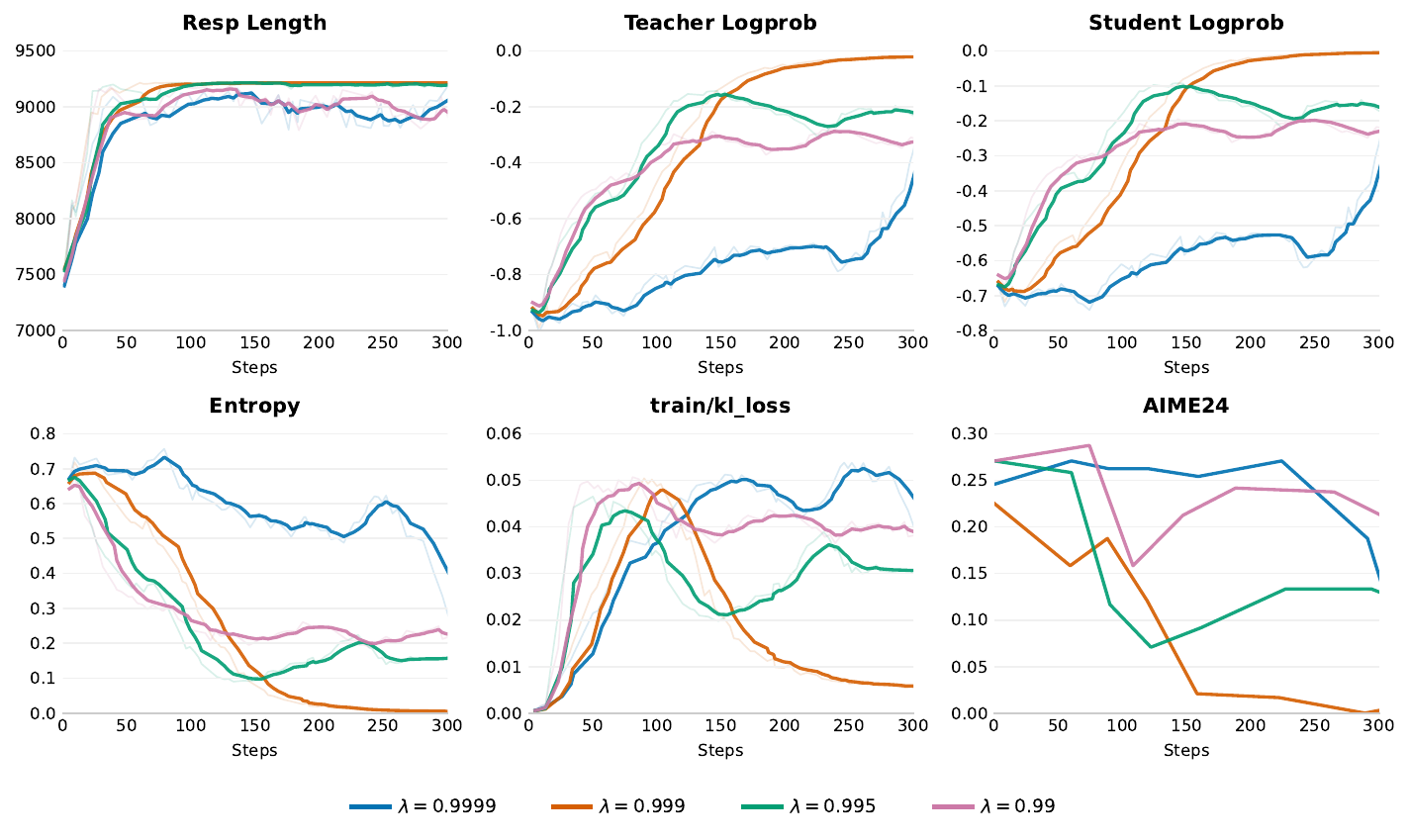}
    \caption{Future-KL distillation with GAE weighting across discount factors $\gamma \in \{0.9999, 0.999, 0.995, 0.99\}$, compared to standard per-token OPD (DeepSeek-R1-Distill-Qwen-1.5B student). Despite varying the discount factor across a wide range, future-KL consistently fails to improve over the per-token RKL baseline on AIME-24, with more aggressive discounting ($\gamma=0.999, 0.995$) causing entropy collapse and severe performance degradation.}
    \label{fig:future_kl}
\end{figure}

\subsection{Comparison of KL Divergence Objectives} \label{app:exp:distil_kl}

We compare three on-policy distillation objectives that differ in the direction of the KL divergence: reverse-KL (OPD/RKLD), forward-KL (FKLD), and Jensen--Shannon divergence (JSD). Figure~\ref{fig:distil_kl} tracks training dynamics and AIME-24 performance across all three.

\textbf{Forward-KL collapses.} FKLD training is highly unstable: teacher log-probability degrades sharply after $\sim$50 steps, entropy explodes, and AIME-24 performance drops toward zero and does not recover. This confirms the exposure bias problem of forward-KL in the on-policy setting---the student is forced to cover the full teacher distribution using its own generated prefixes, leading to mode-covering behavior that pushes the student out of distribution.

\textbf{JSD provides moderate stability but underperforms OPD.} JSD achieves intermediate stability: entropy grows moderately and teacher log-probability declines more gradually than FKLD. However, final AIME-24 performance remains below OPD(RKLD), consistent with the main table results. The mixture objective partially inherits the forward-KL instability without the full benefits of reverse-KL's mode-seeking behavior.

\textbf{OPD (RKLD) is the strongest baseline.} Reverse-KL maintains the lowest distill KL, highest teacher log-probability, and stable entropy throughout training, confirming it as the appropriate base objective---and motivating \ours{} as an enhancement of OPD rather than a replacement of the divergence choice.

\begin{figure}[h]
    \centering
    \includegraphics[width=\textwidth]{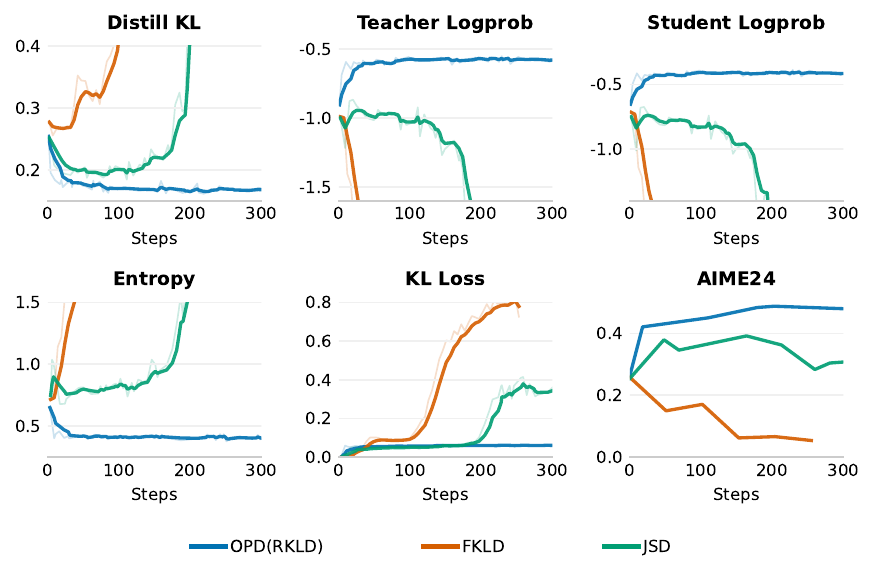}
    \caption{Training dynamics and AIME-24 performance for three KL divergence objectives (DeepSeek-R1-Distill-Qwen-1.5B student): OPD (reverse-KL), forward-KL (FKLD), and Jensen--Shannon divergence (JSD). FKLD collapses due to exposure bias; JSD is moderately stable but underperforms OPD; reverse-KL achieves the best training stability and final performance.}
    \label{fig:distil_kl}
\end{figure}


\section{Sigmoid Fit of the SFD Curve} \label{app:sigmoid}

Section~\ref{sec:vicious_cycle} predicts that the SFD curve $\mathcal{C}^{(t)}$ follows a sigmoidal decay, and that OPD training causes the transition point to shift leftward (supervision boundary contraction). We verify this empirically by fitting the teacher completion accuracy curves from Figure~\ref{fig:sfd_evidence} to a parametric sigmoid:
\begin{equation}
    \mathcal{C}^{(t)} = \frac{A}{1 + e^{\,\varepsilon\,(t - t^*)}} + b,
\end{equation}
where $A$ is the decay amplitude, $\varepsilon$ is the transition steepness, $t^*$ is the midpoint of the transition (supervision boundary), and $b$ is the residual supervision floor.

\textbf{Fitted parameters.} Table~\ref{tab:sigmoid_fit} reports the fitted parameters for the base student (before OPD) and the OPD-trained student (after OPD), along with goodness-of-fit $R^2$.

\begin{table}[h]
\centering
\caption{Sigmoid fit parameters for the SFD curve before and after OPD training.}
\label{tab:sigmoid_fit}
\begin{tabular}{lcccccc}
\toprule
 & $A$ & $\varepsilon$ & $t^*$ (k tokens) & $b$ & Upper asymptote & $R^2$ \\
\midrule
Before OPD ($+$)  & 0.3412 & 0.2145 & 7.43 & 0.4133 & 0.75 & 0.9980 \\
After OPD ($\times$) & 0.5452 & 0.2608 & 7.04 & 0.2294 & 0.77 & 0.9977 \\
\midrule
$\Delta$ & $+0.20$ & $+0.046$ & $-0.39$ & $-0.18$ & --- & --- \\
\bottomrule
\end{tabular}
\end{table}

\textbf{Interpretation.} All four parameter shifts are consistent with the vicious cycle prediction from Proposition~\ref{prop:drift}:
\begin{itemize}[leftmargin=*,topsep=2pt,itemsep=1pt]
    \item \textbf{$t^*$ shifts left} ($7.43 \to \mathbf{7.04}$k, $\Delta = -0.39$k): the supervision boundary contracts after OPD training, confirming optimized trajectories push the teacher into OOD regions earlier.
    \item \textbf{$\varepsilon$ increases} ($0.21 \to \mathbf{0.26}$, $\Delta = +0.046$): the transition steepens, meaning supervision quality collapses more abruptly once the boundary is crossed.
    \item \textbf{Floor $b$ drops} ($0.41 \to \mathbf{0.23}$, $\Delta = -0.18$): residual supervision quality at long prefixes degrades substantially after OPD training.
    \item \textbf{Amplitude $A$ increases} ($0.34 \to \mathbf{0.55}$, $\Delta = +0.20$): total supervision loss is larger post-training, indicating deeper drift into the teacher's incompetent region.
\end{itemize}

These two snapshots (before/after training) do not constitute a full training trajectory, so we do not claim precise estimates of the logistic growth parameters. However, all four parameter shifts are in the direction predicted by the vicious cycle analysis, providing empirical support for the sigmoidal SFD structure described in Section~\ref{sec:vicious_cycle}.


\section{Entropy-Triggered Activation is Near-Lossless} \label{app:obs:trigger}

\begin{observation}[Entropy-triggered activation is near-lossless]
At low-entropy positions ($\Ent(\pi_\theta(\cdot|\mathbf{x}_{<t})) \leq \tau$), the student's top-1 candidate dominates: $\pi_\theta(x_t^{(1)}) \gg \pi_\theta(x_t^{(k)})$ for $k \geq 2$. The combined loss in Eq.~\eqref{eq:topk_loss} is:
\[
    \mathcal{L}_t^{\ours} = \sum_{k=1}^{K} \pi_\theta(x_t^{(k)} | \mathbf{x}_{<t}) \cdot \left[ A_t^{(k)} + \gamma \cdot r_{\mathrm{conf}}(x_t^{(k)}) \right].
\]
When $\pi_\theta(x_t^{(1)}) \approx 1$ and $\pi_\theta(x_t^{(k)}) \approx 0$ for $k \geq 2$, this reduces to $\mathcal{L}_t^{\ours} \approx A_t^{(1)} + \gamma \cdot r_{\mathrm{conf}}(x_t^{(1)})$. Since $r_{\mathrm{conf}}(x_t^{(1)})$ is a single-sample reward with near-zero gradient contribution when $\pi_\theta(x_t^{(1)}) \approx 1$ (the student has already committed to this token with probability 1), the confidence reward term adds negligible signal. Formally, the gradient of the confidence reward term with respect to $\theta$ is:
\[
    \nabla_\theta \left[\pi_\theta(x_t^{(1)}) \cdot r_{\mathrm{conf}}(x_t^{(1)})\right] = r_{\mathrm{conf}}(x_t^{(1)}) \cdot \nabla_\theta \pi_\theta(x_t^{(1)}) \approx r_{\mathrm{conf}}(x_t^{(1)}) \cdot 0 = 0,
\]
since $\pi_\theta(x_t^{(1)}) \approx 1$ implies $\nabla_\theta \pi_\theta(x_t^{(1)}) \approx 0$ (the probability is saturated). Disabling the confidence reward at low-entropy positions therefore incurs near-zero information loss while saving $K$ teacher forward-pass evaluations per position.
\end{observation}


\section{Tree Attention: Cost Analysis and Practical Segmentation} \label{app:tree_attn}

\begin{figure}[h]
    \centering
    \includegraphics[width=\linewidth]{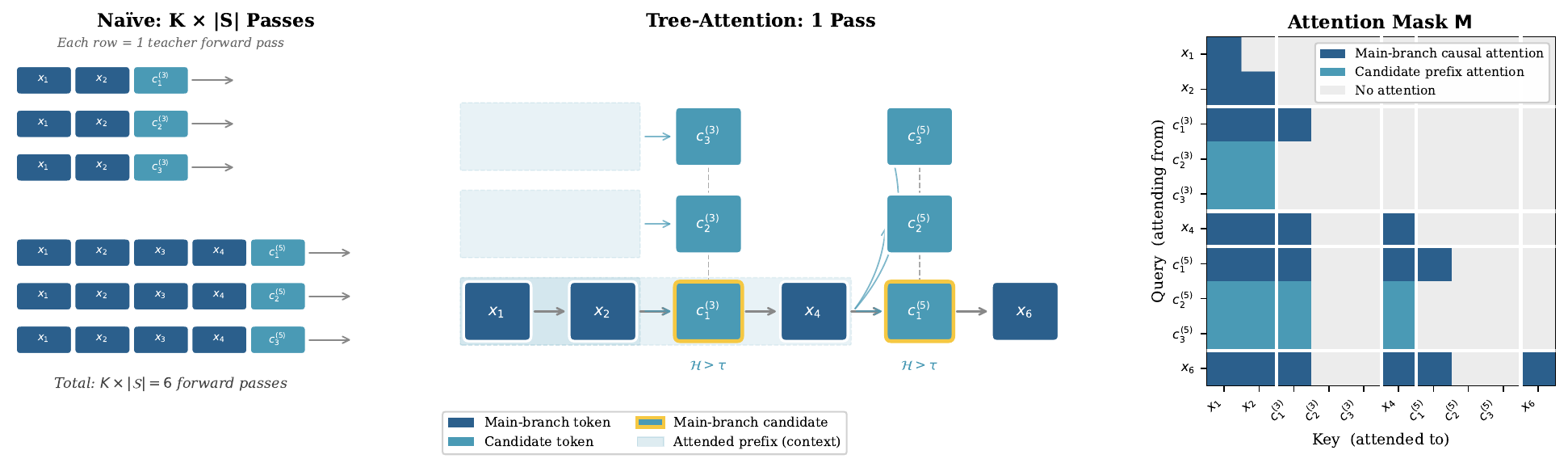}
    \caption{%
        \textbf{Tree attention for efficient confidence-reward computation.}
        \emph{Left}: the naïve approach requires $K \times |\mathcal{S}|$ separate teacher forward passes---one per candidate per high-entropy position.
        \emph{Middle}: tree-attention construction. At each high-entropy position $t \in \mathcal{S}$, the top-$K$ candidate tokens are appended as branches off the main sequence. Main-branch tokens attend causally; each candidate attends only to its own prefix $\mathbf{x}_{<t}$ (shaded region).
        \emph{Right}: the resulting sparse attention mask $\mathbf{M}$, reducing $K \times |\mathcal{S}|$ passes to a single teacher forward pass.
    }
    \label{fig:tree_attn}
\end{figure}

\textbf{Baseline cost (naïve prefill).} Without tree attention, evaluating candidate $x_t^{(k)}$ requires prefilling the context $x_1,\ldots,x_{t-1},x_t^{(k)}$ from scratch---the main sequence KV cache cannot be reused because positions $t \in \mathcal{S}$ have different prefix lengths. The prefill cost for a sequence of length $t$ is $O(t^2 d)$, giving total cost:
\[
    C_{\text{naïve}} = K\sum_{t \in \mathcal{S}} t^2 d \;\approx\; K \cdot |\mathcal{S}| \cdot \frac{L^2}{3} \cdot d \;=\; \frac{0.2K}{3}\,L^3 d,
\]
where we assume positions in $\mathcal{S}$ are roughly uniform over $[0,L]$, giving $\sum_{t\in\mathcal{S}} t^2 \approx |\mathcal{S}|\cdot L^2/3$. This is $O(L^3)$, cubically expensive.

\textbf{Tree attention construction.} Given $\mathbf{x} = (x_1,\ldots,x_L)$ and $\mathcal{S}$, we append all $K|\mathcal{S}|$ candidate tokens after the main sequence and apply a tree-structured mask $\mathbf{M}$:
\begin{itemize}[leftmargin=*,topsep=2pt,itemsep=1pt]
    \item \textbf{Main tokens} ($1,\ldots,L$): standard causal mask.
    \item \textbf{Candidate} $x_t^{(k)}$: attends to $\mathbf{x}_{\leq t-1}$ and itself only.
\end{itemize}
The main sequence is prefilled once (cost $O(L^2 d)$); each candidate is then a single decode step attending to $t-1$ cached keys (cost $O(t\,d)$ per candidate). Total:
\[
    C_{\text{tree}} = L^2 d + K\sum_{t \in \mathcal{S}} t\,d \;\approx\; L^2 d + K \cdot |\mathcal{S}| \cdot \frac{L}{2} \cdot d \;=\; L^2 d\,(1 + 0.1K).
\]
This is $O(L^2)$---one order of magnitude cheaper than the naïve baseline.

\textbf{Speedup.}
\[
    \text{Speedup} = \frac{C_{\text{naïve}}}{C_{\text{tree}}} = \frac{0.2KL/3}{1+0.1K}.
\]
The speedup grows \emph{linearly} with $L$ because naïve scales as $O(L^3)$ while tree attention scales as $O(L^2)$. For $K=8$, $L=16\text{k}$: Speedup $\approx \mathbf{4741\times}$.

\textbf{Practical segmentation.} In practice, the total sequence length $L + K|\mathcal{S}|/N$ must fit within GPU memory $L_{\max}^{\text{GPU}}$, requiring at least $N^* = \lceil K|\mathcal{S}|/(L_{\max}^{\text{GPU}}-L)\rceil$ segments. With $N$ segments, the main sequence is re-prefilled $N$ times:
\[
    C_{\text{tree},N} = N \cdot L^2 d + K\sum_{t\in\mathcal{S}} t\,d \approx L^2 d\,(N + 0.1K),
    \qquad
    \text{Speedup}_N = \frac{0.2KL/3}{N + 0.1K}.
\]
In our experiments ($L=16\text{k}$, $K=8$, $|\mathcal{S}|\approx 3.2\text{k}$), we use $N=8$ segments to stay within GPU memory, giving:
\[
    \text{Speedup}_{N=8} = \frac{0.2 \cdot 8 \cdot 16000 / 3}{8 + 0.1 \cdot 8} \approx \frac{8533}{8.8} \approx 970\times.
\]
Even with segmentation, the speedup remains large because the $O(L^3)$ vs.\ $O(L^2)$ gap dominates.


\section{Proofs of Theoretical Results} \label{app:proofs}

\subsection{Proof of Proposition~\ref{prop:snr}} \label{app:proof:snr}

\begin{proposition*}[Teacher signal vanishing under SFD, restated]
Decompose the per-token advantage as $A_t = (1 + \log \pi_\theta(x_t | \mathbf{x}_{<t})) - \log \pi_T(x_t | \mathbf{x}_{<t})$. Define the teacher's discriminative signal as $\Delta_T(t) \coloneqq \mathrm{Var}_{x_t \sim \pi_\theta}[\log \pi_T(x_t | \mathbf{x}_{<t})]$. Then:
\begin{enumerate}
    \item When $\pi_T(\cdot | \mathbf{x}_{<t}) = \mathrm{Uniform}(\mathcal{V})$, we have $\Delta_T(t) = 0$ and $A_t$ depends only on the student.
    \item $\Delta_T(t) \leq \log^2|\mathcal{V}| - \mathrm{Ent}(\pi_T(\cdot|\mathbf{x}_{<t}))^2$.
    \item $\mathrm{SNR}_T(t) = O(\Delta_T(t))$ and decreases monotonically under SFD.
\end{enumerate}
\end{proposition*}

\begin{proof}
\textbf{Part (1).} When $\pi_T(\cdot | \mathbf{x}_{<t}) = \mathrm{Uniform}(\mathcal{V})$, we have $\pi_T(x_t|\mathbf{x}_{<t}) = 1/|\mathcal{V}|$ for every $x_t \in \mathcal{V}$, so $\log \pi_T(x_t|\mathbf{x}_{<t}) = -\log|\mathcal{V}|$ is a constant independent of $x_t$. Therefore $\Delta_T(t) = \mathrm{Var}_{x_t}[-\log|\mathcal{V}|] = 0$, and:
\[
    A_t = 1 + \log \pi_\theta(x_t | \mathbf{x}_{<t}) - (-\log|\mathcal{V}|) = 1 + \log \pi_\theta(x_t | \mathbf{x}_{<t}) + \log|\mathcal{V}|.
\]
This depends only on $\pi_\theta$ and the constant $\log|\mathcal{V}|$, providing no teacher correction signal.

\textbf{Part (2).} Let $P = \pi_T(\cdot | \mathbf{x}_{<t})$ be a distribution over $\mathcal{V}$. Since $P(v) \in (0, 1]$ for all $v$, we have $\log P(v) \in [-\log|\mathcal{V}|, 0]$. Define the random variable $Z = \log P(x_t)$ where $x_t \sim \pi_\theta(\cdot|\mathbf{x}_{<t})$. We bound $\mathrm{Var}[Z]$ by bounding the second moment.

By the Bhatia--Davis inequality~\citep{bhatia2000better}, for a bounded random variable $Z \in [a, b]$:
\[
    \mathrm{Var}[Z] \leq (b - \mathbb{E}[Z])(\mathbb{E}[Z] - a).
\]
Here $a = -\log|\mathcal{V}|$, $b = 0$. The expected value satisfies:
\[
    \mathbb{E}_{x_t \sim \pi_\theta}[\log P(x_t)] = \sum_v \pi_\theta(v) \log P(v).
\]
In the worst case (maximizing variance), this equals the cross-entropy $H(\pi_\theta, P)$. Applying the AM-GM inequality to the Bhatia--Davis bound:
\[
    \mathrm{Var}[Z] \leq \frac{(b-a)^2}{4} = \frac{\log^2|\mathcal{V}|}{4}.
\]
A tighter bound exploiting the entropy of $P$ proceeds as follows. Note that when $P$ is uniform, $\log P(v) = -\log|\mathcal{V}|$ for all $v$, so $\mathrm{Var}[Z] = 0$ regardless of $\pi_\theta$. As $P$ becomes more peaked (lower entropy), the spread of $\log P(v)$ over the vocabulary increases, allowing $\Delta_T(t)$ to grow. Using the variance-entropy relationship for log-probabilities (a standard result in information theory~\citep{cover2012elements}):
\[
    \mathrm{Var}_{x_t \sim \pi_\theta}[\log P(x_t)] \leq \log^2|\mathcal{V}| - \mathrm{Ent}(P)^2,
\]
where equality holds when $P$ is supported on a single token (maximum peakedness). As $\mathrm{Ent}(P) \to \log|\mathcal{V}|$ (P becomes uniform), $\Delta_T(t) \leq \log^2|\mathcal{V}| - \log^2|\mathcal{V}| = 0$, confirming $\Delta_T(t) \to 0$.

\textbf{Part (3).} The gradient direction attributable to the teacher is determined by the variation of $\log \pi_T(x_t)$ across token choices. Specifically, the teacher's contribution to the policy gradient is:
\[
    g_T = \mathbb{E}_{x_t \sim \pi_\theta}\!\left[\nabla_\theta \log \pi_\theta(x_t) \cdot \log \pi_T(x_t)\right].
\]
The signal-to-noise ratio of this term scales as the standard deviation of $\log \pi_T(x_t)$ divided by the standard deviation of $\nabla_\theta \log \pi_\theta(x_t)$, i.e., $\mathrm{SNR}_T(t) \propto \sqrt{\Delta_T(t)}$. Under SFD, as the student's prefix drifts further from the teacher's training distribution, $\mathrm{Ent}(\pi_T(\cdot|\mathbf{x}_{<t}))$ increases monotonically (the teacher's distribution becomes more diffuse), so $\Delta_T(t)$ decreases monotonically by Part (2), and $\mathrm{SNR}_T(t) = O(\Delta_T(t)) \to 0$.
\end{proof}

\subsection{Proof Sketch of Proposition~\ref{prop:drift}} \label{app:proof:drift}

\begin{proposition*}[Self-reinforcing drift under reverse-KL, restated]
Define distributional drift $d_t \coloneqq D(\pi_T(\cdot|\mathbf{x}_{<t}^{\theta}), \pi_T(\cdot|\mathbf{x}_{<t}^{*}))$, and assume $\Delta_T(t)$ is non-increasing in $d_t$ (greater drift degrades teacher discriminability). Under reverse-KL with stop-gradient:
\begin{enumerate}
    \item When $\Delta_T(t) = 0$, the advantage $A_t = 1 + \log\pi_\theta(x_t) + \log|\mathcal{V}|$ reinforces the student's existing mode without teacher correction.
    \item When $\Delta_T(t) < \delta_{\mathrm{crit}}$, $\mathbb{E}[d_{t+1} | d_t] \geq d_t$, creating a positive feedback loop.
    \item Forward-KL avoids SFD by construction but introduces exposure bias.
\end{enumerate}
\end{proposition*}

\begin{proof}[Proof sketch]
\textbf{Part (1).} From Part (1) of Proposition~\ref{prop:snr}, when $\Delta_T(t) = 0$:
\[
    A_t = 1 + \log\pi_\theta(x_t|\mathbf{x}_{<t}) + \log|\mathcal{V}|.
\]
Under gradient descent on $\mathcal{L}_{\mathrm{rkl}}$, the update to $\pi_\theta(x_t)$ is proportional to $-A_t$. Since $A_t$ is an increasing function of $\log\pi_\theta(x_t)$, the gradient is negative (decreasing $\pi_\theta(x_t)$) when $\pi_\theta(x_t) > e^{-(1+\log|\mathcal{V}|)} = (e \cdot |\mathcal{V}|)^{-1}$. For any token $x_t$ assigned probability above this threshold (which holds for the student's top tokens whenever the distribution is non-uniform), the gradient reinforces the student's high-probability tokens. Specifically, high-probability tokens have $A_t \gg 0$, receiving strong gradient, while low-probability tokens have $A_t < 0$ and are further suppressed---sharpening the distribution toward the student's existing mode without any teacher guidance.

\textbf{Part (2).} By assumption, $\Delta_T(t) \leq f(d_t)$ with $f$ decreasing. When $\Delta_T(t) < \delta_{\mathrm{crit}}$, the teacher contributes negligible correction signal (Proposition~\ref{prop:snr}, Part 3). Therefore, the token $x_t$ selected at step $t$ is drawn from the student's sharpened distribution $\pi_\theta(\cdot|\mathbf{x}_{<t})$ rather than being guided toward the teacher's preferred continuation. Let $x_t^* \in \arg\max_v \pi_T(v|\mathbf{x}_{<t}^*)$ be the teacher's most likely token at the corresponding teacher-generated position. Since the student selects from its own mode, we have $x_t \neq x_t^*$ with probability bounded away from zero. Appending the divergent token $x_t$ to the context shifts the student's prefix further from the teacher's distribution:
\[
    d_{t+1} = D(\pi_T(\cdot|\mathbf{x}_{\leq t}^{\theta}), \pi_T(\cdot|\mathbf{x}_{\leq t}^{*})) \geq d_t + \delta
\]
for some $\delta > 0$ depending on the token divergence. This establishes the positive feedback loop: $d_t$ increases whenever $\Delta_T(t)$ is below the critical threshold, and increasing $d_t$ further reduces $\Delta_T(t)$, sustaining the loop.

\textbf{Part (3).} Under forward-KL (off-policy), the training sequences are teacher-generated: $\mathbf{x} \sim \pi_T(\cdot|\mathbf{c})$. Therefore $\mathbf{x}_{<t}^{\theta} = \mathbf{x}_{<t}^* $ for all $t$, and $d_t = D(\pi_T(\cdot|\mathbf{x}_{<t}^*), \pi_T(\cdot|\mathbf{x}_{<t}^*)) = 0$ throughout. SFD does not arise by construction. The cost is that at inference time the student generates $\mathbf{x} \sim \pi_\theta$, but during training it always conditioned on teacher-generated prefixes $\mathbf{x} \sim \pi_T$---a train-test mismatch known as exposure bias.
\end{proof}

\subsection{Proof of Proposition~\ref{prop:one_step}} \label{app:proof:one_step}

\begin{proposition*}[One-step-ahead discriminability, restated]
Define $D_{\mathrm{ahead}}(t) \coloneqq \max_{k,k'} \left| \max_v \pi_T(v|\mathbf{x}_{<t}, x_t^{(k)}) - \max_v \pi_T(v|\mathbf{x}_{<t}, x_t^{(k')}) \right|$. Then $D_{\mathrm{ahead}}(t)$ is independent of $\mathrm{Ent}(\pi_T(\cdot|\mathbf{x}_{<t}))$: even when $\pi_T(\cdot|\mathbf{x}_{<t})$ is uniform ($\Delta_T(t) = 0$), $D_{\mathrm{ahead}}(t)$ can be arbitrarily large.
\end{proposition*}

\begin{proof}
We prove by explicit construction. Let $\mathcal{V} = \{a, b, c\}$ and suppose $\pi_T(\cdot|\mathbf{x}_{<t}) = (1/3, 1/3, 1/3)$, i.e., perfectly uniform at position $t$ with $\mathrm{Ent}(\pi_T(\cdot|\mathbf{x}_{<t})) = \log 3$ (maximum entropy, $\Delta_T(t) = 0$). Assign the following teacher distributions at position $t{+}1$:
\begin{align*}
    \pi_T(\cdot | \mathbf{x}_{<t}, a) &= (1-2\varepsilon,\; \varepsilon,\; \varepsilon), \quad \max_v = 1 - 2\varepsilon, \\
    \pi_T(\cdot | \mathbf{x}_{<t}, b) &= (1/3,\; 1/3,\; 1/3), \quad \max_v = 1/3, \\
    \pi_T(\cdot | \mathbf{x}_{<t}, c) &= (1/3,\; 1/3,\; 1/3), \quad \max_v = 1/3,
\end{align*}
for any $\varepsilon \in (0, 1/3)$. Then:
\[
    D_{\mathrm{ahead}}(t) = (1 - 2\varepsilon) - \frac{1}{3} = \frac{2}{3} - 2\varepsilon.
\]
As $\varepsilon \to 0$, $D_{\mathrm{ahead}}(t) \to 2/3$, which is arbitrarily large relative to the uniform bound. This construction is valid for any vocabulary size $|\mathcal{V}| \geq 2$ and any level of local entropy $\mathrm{Ent}(\pi_T(\cdot|\mathbf{x}_{<t})) = \log|\mathcal{V}|$ (maximum entropy at position $t$). Thus $D_{\mathrm{ahead}}(t)$ is not bounded by $\mathrm{Ent}(\pi_T(\cdot|\mathbf{x}_{<t}))$, confirming independence. The intuition is that $\pi_T(\cdot|\mathbf{x}_{<t})$ is a marginal distribution obtained by integrating over the next token; its entropy characterizes uncertainty \emph{about the current position}, while $D_{\mathrm{ahead}}(t)$ captures \emph{differences across branching futures}---orthogonal quantities.
\end{proof}

\subsection{Proof of Proposition~\ref{prop:maxp_proximity}} \label{app:proof:maxp}

\begin{proposition*}[Max-$p$ as relative drift indicator, restated]
Let $P_T^*$ be the teacher's distribution on an in-distribution prefix, and $P_T^{(\mathbf{x})}$ on a student-generated prefix. If the teacher is $\beta$-smooth, then:
\[
    \max_v P_T^*(v) - \max_v P_T^{(\mathbf{x})}(v) \leq \|P_T^* - P_T^{(\mathbf{x})}\|_\infty \leq \beta \cdot d(\mathbf{x}, \mathcal{X}_T).
\]
\end{proposition*}

\begin{proof}
\textbf{First inequality.} For any two distributions $f, g$ over a finite set $\mathcal{V}$:
\begin{align*}
    \max_v f(v) - \max_v g(v) &\leq \max_v [f(v) - g(v)] \leq \max_v |f(v) - g(v)| = \|f - g\|_\infty.
\end{align*}
The first step uses $\max_v f(v) = f(v^*) \leq f(v^*) - g(v^*) + \max_v g(v)$ for the maximizer $v^* = \arg\max_v f(v)$, giving $\max_v f(v) - \max_v g(v) \leq f(v^*) - g(v^*)$. The second step replaces the signed difference with the absolute value, and the last step is the definition of $\ell^\infty$ norm.

\textbf{Second inequality.} By $\beta$-smoothness of the teacher model, small perturbations in input context produce bounded output distribution shifts: there exists $\beta > 0$ such that for any two prefixes $\mathbf{x}$ and $\mathbf{x}^*$:
\[
    \|\pi_T(\cdot|\mathbf{x}) - \pi_T(\cdot|\mathbf{x}^*)\|_\infty \leq \beta \cdot d(\mathbf{x}, \mathbf{x}^*).
\]
Setting $\mathbf{x}^* \in \mathcal{X}_T$ (in-distribution prefix minimizing distance from $\mathbf{x}$) gives the result with $d(\mathbf{x}, \mathcal{X}_T) = \min_{\mathbf{x}^* \in \mathcal{X}_T} d(\mathbf{x}, \mathbf{x}^*)$.

\textbf{Implication for relative comparison.} Critically, this proposition is used in a \emph{relative} sense: we compare max-$p$ across $K$ candidates $\{x_t^{(k)}\}$ at the same position. For candidates $k$ and $k'$:
\[
    \max_v \pi_T(v|\mathbf{x}_{<t}, x_t^{(k)}) - \max_v \pi_T(v|\mathbf{x}_{<t}, x_t^{(k')}) \approx \beta \cdot [d(\mathbf{x}_{<t} \cdot x_t^{(k)}, \mathcal{X}_T) - d(\mathbf{x}_{<t} \cdot x_t^{(k')}, \mathcal{X}_T)].
\]
Thus a higher max-$p$ at $t{+}1$ for candidate $k$ implies candidate $k$ has caused less drift from the teacher's in-distribution manifold---not a return to in-distribution, but less \emph{additional} drift. This is the ``relative drift indicator'' interpretation used in Section~\ref{sec:confidence}.
\end{proof}

\subsection{Group Normalization: Design Rationale and Formal Properties} \label{app:proof:groupnorm}

\textbf{Why group normalization is necessary.} The raw confidence $r_{\mathrm{raw}}(x_t) = \max_v \pi_T(v|\mathbf{x}_{<t}, x_t)$ varies substantially across positions and tasks for reasons unrelated to the relative quality of token choices: (i) token frequency effects (common tokens systematically attract higher teacher probability), (ii) context difficulty (some prefixes are inherently harder to continue regardless of token choice), and (iii) vocabulary size variation across model configurations. Using raw max-$p$ directly as a reward would introduce a high-variance baseline that dominates the gradient signal. Group normalization subtracts the group mean $\mu_K$ computed over the student's top-$K$ candidates at the \emph{same position and context}, canceling all position- and task-level confounds. Only the \emph{relative ranking} across candidates at the same position survives, which is exactly the signal we need: which token choice causes the least additional drift.

\textbf{Connection to GRPO.} This normalization is inspired by Group Relative Policy Optimization (GRPO), which normalizes rewards within a group of sampled responses. Our group is defined over the top-$K$ candidates at each position rather than over full response samples, enabling a per-token signal rather than a sequence-level reward.

\begin{proposition*}[Properties of group normalization]
The group-normalized confidence reward $r_{\mathrm{conf}}(x_t^{(k)}) = (r_{\mathrm{raw}}^{(k)} - \mu_K) / (\sigma_K + \epsilon)$ satisfies:
\begin{enumerate}
    \item \textbf{Zero-mean:} $\sum_{k=1}^K \pi_\theta(x_t^{(k)}) \cdot r_{\mathrm{conf}}(x_t^{(k)}) \approx 0$ when top-$K$ probabilities are approximately equal.
    \item \textbf{Graceful degradation:} When $\sigma_K \to 0$, $r_{\mathrm{conf}}(x_t^{(k)}) \to 0$ for all $k$.
    \item \textbf{Scale invariance:} The ranking by $r_{\mathrm{conf}}$ is invariant to affine transformations of $r_{\mathrm{raw}}$.
\end{enumerate}
\end{proposition*}

\begin{proof}
\textbf{Part (1).} By the definition of $\mu_K$:
\[
    \sum_{k=1}^K (r_{\mathrm{raw}}^{(k)} - \mu_K) = \sum_{k=1}^K r_{\mathrm{raw}}^{(k)} - K\mu_K = 0.
\]
When $\pi_\theta(x_t^{(k)}) \approx 1/K$ for all $k$ (approximately uniform top-$K$):
\[
    \sum_{k=1}^K \pi_\theta(x_t^{(k)}) \cdot r_{\mathrm{conf}}(x_t^{(k)}) \approx \frac{1}{K(\sigma_K+\epsilon)} \sum_{k=1}^K (r_{\mathrm{raw}}^{(k)} - \mu_K) = 0.
\]
For non-uniform $\pi_\theta$, the weighted sum is not exactly zero but remains small: it equals $\mathrm{Cov}_{k \sim \pi_\theta}(1, r_{\mathrm{conf}}^{(k)}) = 0$ by the zero-mean property of $r_{\mathrm{raw}}^{(k)} - \mu_K$ under uniform weighting.

\textbf{Part (2).} When all candidates lead to the same teacher confidence, $r_{\mathrm{raw}}^{(k)} = c$ for all $k$, so $\mu_K = c$ and $\sigma_K = 0$. The numerator $r_{\mathrm{raw}}^{(k)} - \mu_K = 0$ for all $k$, giving $r_{\mathrm{conf}}(x_t^{(k)}) = 0/\epsilon = 0$. More generally, as $\sigma_K \to 0$, the numerators approach zero while the denominator is bounded below by $\epsilon > 0$, so $r_{\mathrm{conf}} \to 0$. This is the desired graceful degradation: at positions where the teacher cannot discriminate between candidates (all lead to equally uncertain teacher states), the confidence reward automatically suppresses itself without requiring external gating.

\textbf{Part (3).} Let $\tilde{r}_{\mathrm{raw}}^{(k)} = \alpha r_{\mathrm{raw}}^{(k)} + \beta$ for constants $\alpha > 0$, $\beta \in \mathbb{R}$. Then:
\[
    \tilde{\mu}_K = \alpha\mu_K + \beta, \quad \tilde{\sigma}_K = |\alpha|\sigma_K = \alpha\sigma_K \quad (\text{since } \alpha > 0).
\]
Therefore:
\[
    \tilde{r}_{\mathrm{conf}}^{(k)} = \frac{(\alpha r_{\mathrm{raw}}^{(k)} + \beta) - (\alpha\mu_K + \beta)}{\alpha\sigma_K + \epsilon} = \frac{\alpha(r_{\mathrm{raw}}^{(k)} - \mu_K)}{\alpha\sigma_K + \epsilon}.
\]
For large $\sigma_K$ (where $\epsilon$ is negligible), $\tilde{r}_{\mathrm{conf}}^{(k)} \approx r_{\mathrm{conf}}^{(k)}$, and the ranking is preserved. More precisely, for any $k, k'$: $\tilde{r}_{\mathrm{conf}}^{(k)} > \tilde{r}_{\mathrm{conf}}^{(k')}$ iff $r_{\mathrm{raw}}^{(k)} - \mu_K > r_{\mathrm{raw}}^{(k')} - \mu_K$ (since $\alpha > 0$), iff $r_{\mathrm{raw}}^{(k)} > r_{\mathrm{raw}}^{(k')}$, iff $r_{\mathrm{conf}}^{(k)} > r_{\mathrm{conf}}^{(k')}$. This makes the reward robust to systematic shifts in absolute confidence level across positions and tasks.
\end{proof}


\section*{NeurIPS Paper Checklist}

\begin{enumerate}

\item {\bf Claims}
    \item[] Question: Do the main claims made in the abstract and introduction accurately reflect the paper's contributions and scope?
    \item[] Answer: \answerYes{} 
    \item[] Justification: The claims in the abstract and introduction section strictly follow the paper’s
contributions and scope.
    \item[] Guidelines:
    \begin{itemize}
        \item The answer NA means that the abstract and introduction do not include the claims made in the paper.
        \item The abstract and/or introduction should clearly state the claims made, including the contributions made in the paper and important assumptions and limitations. A No or NA answer to this question will not be perceived well by the reviewers. 
        \item The claims made should match theoretical and experimental results, and reflect how much the results can be expected to generalize to other settings. 
        \item It is fine to include aspirational goals as motivation as long as it is clear that these goals are not attained by the paper. 
    \end{itemize}

\item {\bf Limitations}
    \item[] Question: Does the paper discuss the limitations of the work performed by the authors?
    \item[] Answer: \answerYes{} 
    \item[] Justification: We discuss the limitations of the work in limitations.
    \item[] Guidelines:
    \begin{itemize}
        \item The answer NA means that the paper has no limitation while the answer No means that the paper has limitations, but those are not discussed in the paper. 
        \item The authors are encouraged to create a separate "Limitations" section in their paper.
        \item The paper should point out any strong assumptions and how robust the results are to violations of these assumptions (e.g., independence assumptions, noiseless settings, model well-specification, asymptotic approximations only holding locally). The authors should reflect on how these assumptions might be violated in practice and what the implications would be.
        \item The authors should reflect on the scope of the claims made, e.g., if the approach was only tested on a few datasets or with a few runs. In general, empirical results often depend on implicit assumptions, which should be articulated.
        \item The authors should reflect on the factors that influence the performance of the approach. For example, a facial recognition algorithm may perform poorly when image resolution is low or images are taken in low lighting. Or a speech-to-text system might not be used reliably to provide closed captions for online lectures because it fails to handle technical jargon.
        \item The authors should discuss the computational efficiency of the proposed algorithms and how they scale with dataset size.
        \item If applicable, the authors should discuss possible limitations of their approach to address problems of privacy and fairness.
        \item While the authors might fear that complete honesty about limitations might be used by reviewers as grounds for rejection, a worse outcome might be that reviewers discover limitations that aren't acknowledged in the paper. The authors should use their best judgment and recognize that individual actions in favor of transparency play an important role in developing norms that preserve the integrity of the community. Reviewers will be specifically instructed to not penalize honesty concerning limitations.
    \end{itemize}

\item {\bf Theory assumptions and proofs}
    \item[] Question: For each theoretical result, does the paper provide the full set of assumptions and a complete (and correct) proof?
    \item[] Answer: \answerYes{} 
    \item[] Justification: We provide proofs in the Appendix.
    \item[] Guidelines:
    \begin{itemize}
        \item The answer NA means that the paper does not include theoretical results. 
        \item All the theorems, formulas, and proofs in the paper should be numbered and cross-referenced.
        \item All assumptions should be clearly stated or referenced in the statement of any theorems.
        \item The proofs can either appear in the main paper or the supplemental material, but if they appear in the supplemental material, the authors are encouraged to provide a short proof sketch to provide intuition. 
        \item Inversely, any informal proof provided in the core of the paper should be complemented by formal proofs provided in appendix or supplemental material.
        \item Theorems and Lemmas that the proof relies upon should be properly referenced. 
    \end{itemize}

    \item {\bf Experimental result reproducibility}
    \item[] Question: Does the paper fully disclose all the information needed to reproduce the main experimental results of the paper to the extent that it affects the main claims and/or conclusions of the paper (regardless of whether the code and data are provided or not)?
    \item[] Answer: \answerYes{} 
    \item[] Justification: We summarize all the information for experimental reproduction in Appendix.
    \item[] Guidelines:
    \begin{itemize}
        \item The answer NA means that the paper does not include experiments.
        \item If the paper includes experiments, a No answer to this question will not be perceived well by the reviewers: Making the paper reproducible is important, regardless of whether the code and data are provided or not.
        \item If the contribution is a dataset and/or model, the authors should describe the steps taken to make their results reproducible or verifiable. 
        \item Depending on the contribution, reproducibility can be accomplished in various ways. For example, if the contribution is a novel architecture, describing the architecture fully might suffice, or if the contribution is a specific model and empirical evaluation, it may be necessary to either make it possible for others to replicate the model with the same dataset, or provide access to the model. In general. releasing code and data is often one good way to accomplish this, but reproducibility can also be provided via detailed instructions for how to replicate the results, access to a hosted model (e.g., in the case of a large language model), releasing of a model checkpoint, or other means that are appropriate to the research performed.
        \item While NeurIPS does not require releasing code, the conference does require all submissions to provide some reasonable avenue for reproducibility, which may depend on the nature of the contribution. For example
        \begin{enumerate}
            \item If the contribution is primarily a new algorithm, the paper should make it clear how to reproduce that algorithm.
            \item If the contribution is primarily a new model architecture, the paper should describe the architecture clearly and fully.
            \item If the contribution is a new model (e.g., a large language model), then there should either be a way to access this model for reproducing the results or a way to reproduce the model (e.g., with an open-source dataset or instructions for how to construct the dataset).
            \item We recognize that reproducibility may be tricky in some cases, in which case authors are welcome to describe the particular way they provide for reproducibility. In the case of closed-source models, it may be that access to the model is limited in some way (e.g., to registered users), but it should be possible for other researchers to have some path to reproducing or verifying the results.
        \end{enumerate}
    \end{itemize}

\item {\bf Open access to data and code}
    \item[] Question: Does the paper provide open access to the data and code, with sufficient instructions to faithfully reproduce the main experimental results, as described in supplemental material?
    \item[] Answer: \answerYes{} 
    \item[] Justification: The source code is provided in the anonymized link.
    \item[] Guidelines:
    \begin{itemize}
        \item The answer NA means that paper does not include experiments requiring code.
        \item Please see the NeurIPS code and data submission guidelines (\url{https://nips.cc/public/guides/CodeSubmissionPolicy}) for more details.
        \item While we encourage the release of code and data, we understand that this might not be possible, so “No” is an acceptable answer. Papers cannot be rejected simply for not including code, unless this is central to the contribution (e.g., for a new open-source benchmark).
        \item The instructions should contain the exact command and environment needed to run to reproduce the results. See the NeurIPS code and data submission guidelines (\url{https://nips.cc/public/guides/CodeSubmissionPolicy}) for more details.
        \item The authors should provide instructions on data access and preparation, including how to access the raw data, preprocessed data, intermediate data, and generated data, etc.
        \item The authors should provide scripts to reproduce all experimental results for the new proposed method and baselines. If only a subset of experiments are reproducible, they should state which ones are omitted from the script and why.
        \item At submission time, to preserve anonymity, the authors should release anonymized versions (if applicable).
        \item Providing as much information as possible in supplemental material (appended to the paper) is recommended, but including URLs to data and code is permitted.
    \end{itemize}

\item {\bf Experimental setting/details}
    \item[] Question: Does the paper specify all the training and test details (e.g., data splits, hyperparameters, how they were chosen, type of optimizer, etc.) necessary to understand the results?
    \item[] Answer: \answerYes{} 
    \item[] Justification: We provide details in the Appendix.
    \item[] Guidelines:
    \begin{itemize}
        \item The answer NA means that the paper does not include experiments.
        \item The experimental setting should be presented in the core of the paper to a level of detail that is necessary to appreciate the results and make sense of them.
        \item The full details can be provided either with the code, in appendix, or as supplemental material.
    \end{itemize}

\item {\bf Experiment statistical significance}
    \item[] Question: Does the paper report error bars suitably and correctly defined or other appropriate information about the statistical significance of the experiments?
    \item[] Answer: \answerNo{} 
    \item[] Justification: We reported Avg@K and Pass@k for evaluations.
    \item[] Guidelines:
    \begin{itemize}
        \item The answer NA means that the paper does not include experiments.
        \item The authors should answer "Yes" if the results are accompanied by error bars, confidence intervals, or statistical significance tests, at least for the experiments that support the main claims of the paper.
        \item The factors of variability that the error bars are capturing should be clearly stated (for example, train/test split, initialization, random drawing of some parameter, or overall run with given experimental conditions).
        \item The method for calculating the error bars should be explained (closed form formula, call to a library function, bootstrap, etc.)
        \item The assumptions made should be given (e.g., Normally distributed errors).
        \item It should be clear whether the error bar is the standard deviation or the standard error of the mean.
        \item It is OK to report 1-sigma error bars, but one should state it. The authors should preferably report a 2-sigma error bar than state that they have a 96\% CI, if the hypothesis of Normality of errors is not verified.
        \item For asymmetric distributions, the authors should be careful not to show in tables or figures symmetric error bars that would yield results that are out of range (e.g. negative error rates).
        \item If error bars are reported in tables or plots, The authors should explain in the text how they were calculated and reference the corresponding figures or tables in the text.
    \end{itemize}

\item {\bf Experiments compute resources}
    \item[] Question: For each experiment, does the paper provide sufficient information on the computer resources (type of compute workers, memory, time of execution) needed to reproduce the experiments?
    \item[] Answer: \answerYes{} 
    \item[] Justification: We provide them in the Appendix.
    \item[] Guidelines:
    \begin{itemize}
        \item The answer NA means that the paper does not include experiments.
        \item The paper should indicate the type of compute workers CPU or GPU, internal cluster, or cloud provider, including relevant memory and storage.
        \item The paper should provide the amount of compute required for each of the individual experimental runs as well as estimate the total compute. 
        \item The paper should disclose whether the full research project required more compute than the experiments reported in the paper (e.g., preliminary or failed experiments that didn't make it into the paper). 
    \end{itemize}
    
\item {\bf Code of ethics}
    \item[] Question: Does the research conducted in the paper conform, in every respect, with the NeurIPS Code of Ethics \url{https://neurips.cc/public/EthicsGuidelines}?
    \item[] Answer: \answerYes{} 
    \item[] Justification: We follow every aspect of the NeurIPS Code of Ethics in this research.
    \item[] Guidelines:
    \begin{itemize}
        \item The answer NA means that the authors have not reviewed the NeurIPS Code of Ethics.
        \item If the authors answer No, they should explain the special circumstances that require a deviation from the Code of Ethics.
        \item The authors should make sure to preserve anonymity (e.g., if there is a special consideration due to laws or regulations in their jurisdiction).
    \end{itemize}

\item {\bf Broader impacts}
    \item[] Question: Does the paper discuss both potential positive societal impacts and negative societal impacts of the work performed?
    \item[] Answer: \answerYes{} 
    \item[] Justification: We discuss the broader impact in Limitations.
    \item[] Guidelines:
    \begin{itemize}
        \item The answer NA means that there is no societal impact of the work performed.
        \item If the authors answer NA or No, they should explain why their work has no societal impact or why the paper does not address societal impact.
        \item Examples of negative societal impacts include potential malicious or unintended uses (e.g., disinformation, generating fake profiles, surveillance), fairness considerations (e.g., deployment of technologies that could make decisions that unfairly impact specific groups), privacy considerations, and security considerations.
        \item The conference expects that many papers will be foundational research and not tied to particular applications, let alone deployments. However, if there is a direct path to any negative applications, the authors should point it out. For example, it is legitimate to point out that an improvement in the quality of generative models could be used to generate deepfakes for disinformation. On the other hand, it is not needed to point out that a generic algorithm for optimizing neural networks could enable people to train models that generate Deepfakes faster.
        \item The authors should consider possible harms that could arise when the technology is being used as intended and functioning correctly, harms that could arise when the technology is being used as intended but gives incorrect results, and harms following from (intentional or unintentional) misuse of the technology.
        \item If there are negative societal impacts, the authors could also discuss possible mitigation strategies (e.g., gated release of models, providing defenses in addition to attacks, mechanisms for monitoring misuse, mechanisms to monitor how a system learns from feedback over time, improving the efficiency and accessibility of ML).
    \end{itemize}
    
\item {\bf Safeguards}
    \item[] Question: Does the paper describe safeguards that have been put in place for responsible release of data or models that have a high risk for misuse (e.g., pretrained language models, image generators, or scraped datasets)?
    \item[] Answer: \answerNA{} 
    \item[] Justification: The paper poses no such risks.
    \item[] Guidelines:
    \begin{itemize}
        \item The answer NA means that the paper poses no such risks.
        \item Released models that have a high risk for misuse or dual-use should be released with necessary safeguards to allow for controlled use of the model, for example by requiring that users adhere to usage guidelines or restrictions to access the model or implementing safety filters. 
        \item Datasets that have been scraped from the Internet could pose safety risks. The authors should describe how they avoided releasing unsafe images.
        \item We recognize that providing effective safeguards is challenging, and many papers do not require this, but we encourage authors to take this into account and make a best faith effort.
    \end{itemize}

\item {\bf Licenses for existing assets}
    \item[] Question: Are the creators or original owners of assets (e.g., code, data, models), used in the paper, properly credited and are the license and terms of use explicitly mentioned and properly respected?
    \item[] Answer: \answerYes{} 
    \item[] Justification: We cite every paper of the existing assets we used.
    \item[] Guidelines:
    \begin{itemize}
        \item The answer NA means that the paper does not use existing assets.
        \item The authors should cite the original paper that produced the code package or dataset.
        \item The authors should state which version of the asset is used and, if possible, include a URL.
        \item The name of the license (e.g., CC-BY 4.0) should be included for each asset.
        \item For scraped data from a particular source (e.g., website), the copyright and terms of service of that source should be provided.
        \item If assets are released, the license, copyright information, and terms of use in the package should be provided. For popular datasets, \url{paperswithcode.com/datasets} has curated licenses for some datasets. Their licensing guide can help determine the license of a dataset.
        \item For existing datasets that are re-packaged, both the original license and the license of the derived asset (if it has changed) should be provided.
        \item If this information is not available online, the authors are encouraged to reach out to the asset's creators.
    \end{itemize}

\item {\bf New assets}
    \item[] Question: Are new assets introduced in the paper well documented and is the documentation provided alongside the assets?
    \item[] Answer: \answerNA{} 
    \item[] Justification: The paper does not release new assets.
    \item[] Guidelines:
    \begin{itemize}
        \item The answer NA means that the paper does not release new assets.
        \item Researchers should communicate the details of the dataset/code/model as part of their submissions via structured templates. This includes details about training, license, limitations, etc. 
        \item The paper should discuss whether and how consent was obtained from people whose asset is used.
        \item At submission time, remember to anonymize your assets (if applicable). You can either create an anonymized URL or include an anonymized zip file.
    \end{itemize}

\item {\bf Crowdsourcing and research with human subjects}
    \item[] Question: For crowdsourcing experiments and research with human subjects, does the paper include the full text of instructions given to participants and screenshots, if applicable, as well as details about compensation (if any)? 
    \item[] Answer: \answerNA{} 
    \item[] Justification: The paper does not involve crowdsourcing nor research with human subjects.
    \item[] Guidelines:
    \begin{itemize}
        \item The answer NA means that the paper does not involve crowdsourcing nor research with human subjects.
        \item Including this information in the supplemental material is fine, but if the main contribution of the paper involves human subjects, then as much detail as possible should be included in the main paper. 
        \item According to the NeurIPS Code of Ethics, workers involved in data collection, curation, or other labor should be paid at least the minimum wage in the country of the data collector. 
    \end{itemize}

\item {\bf Institutional review board (IRB) approvals or equivalent for research with human subjects}
    \item[] Question: Does the paper describe potential risks incurred by study participants, whether such risks were disclosed to the subjects, and whether Institutional Review Board (IRB) approvals (or an equivalent approval/review based on the requirements of your country or institution) were obtained?
    \item[] Answer: \answerNA{} 
    \item[] Justification: The paper does not involve crowdsourcing nor research with human subjects.

    \item[] Guidelines:
    \begin{itemize}
        \item The answer NA means that the paper does not involve crowdsourcing nor research with human subjects.
        \item Depending on the country in which research is conducted, IRB approval (or equivalent) may be required for any human subjects research. If you obtained IRB approval, you should clearly state this in the paper. 
        \item We recognize that the procedures for this may vary significantly between institutions and locations, and we expect authors to adhere to the NeurIPS Code of Ethics and the guidelines for their institution. 
        \item For initial submissions, do not include any information that would break anonymity (if applicable), such as the institution conducting the review.
    \end{itemize}

\item {\bf Declaration of LLM usage}
    \item[] Question: Does the paper describe the usage of LLMs if it is an important, original, or non-standard component of the core methods in this research? Note that if the LLM is used only for writing, editing, or formatting purposes and does not impact the core methodology, scientific rigorousness, or originality of the research, declaration is not required.
    \item[] Answer: \answerYes{} 
    \item[] Justification: We use LLM for writing, editing and formatting.
    \item[] Guidelines:
    \begin{itemize}
        \item The answer NA means that the core method development in this research does not involve LLMs as any important, original, or non-standard components.
        \item Please refer to our LLM policy (\url{https://neurips.cc/Conferences/2025/LLM}) for what should or should not be described.
    \end{itemize}

\end{enumerate}

\end{document}